\documentclass[letterpaper, 10 pt, journal, twoside]{IEEEtran}
\usepackage{amsmath,amsfonts}
\usepackage{array}
\usepackage[caption=false,font=normalsize,labelfont=sf,textfont=sf]{subfig}
\usepackage{textcomp}
\usepackage{stfloats}
\usepackage{url}
\usepackage{verbatim}
\usepackage{graphicx}
\usepackage{amsthm} 

\newtheorem{theorem}{Theorem} 
\newtheorem{corollary}{Corollary} 
\def\BibTeX{{\rm B\kern-.05em{\sc i\kern-.025em b}\kern-.08em
    T\kern-.1667em\lower.7ex\hbox{E}\kern-.125emX}}
\usepackage{balance}
\usepackage{amsfonts}
\usepackage{setspace}
\usepackage{parskip}

\usepackage{graphicx}
\usepackage{caption}
\usepackage{amsmath}
\usepackage{lipsum}
\usepackage{subfig}
\usepackage{algorithm}
\usepackage{url}
\usepackage[table]{xcolor}
\usepackage{seqsplit}
\usepackage{comment}
\usepackage[section]{placeins}
\usepackage{arydshln}
\usepackage{placeins}
\usepackage{algpseudocode}
\usepackage{etoolbox}  
\urldef{\footurl}\url{https://github.com/aralab-unr/ESM_ICP}
\makeatletter
\renewcommand{\alglinenumber}[1]{\footnotesize #1:}
\makeatother
 \usepackage[caption=false]{subfig} 
\begin{document}

\title{Registration of 3D Point Sets Using Exponential-based Similarity Matrix}
\author{Ashutosh Singandhupe, Sanket Lokhande, Hung Manh La
\thanks{Corresponding author: Hung La, E-mail: hla@unr.edu}
\thanks{The authors are with the Advanced Robotics and Automation (ARA) Laboratory, Department of Computer Science and Engineering, University of Nevada, Reno, NV 89557, USA}}

\markboth{Journal of \LaTeX\ Class Files,~Vol.~18, No.~9, September~2020}%
{How to Use the IEEEtran \LaTeX \ Templates}

\maketitle

\begin{abstract}
Point cloud registration is a fundamental problem in computer vision and robotics, involving the alignment of 3D point sets captured from varying viewpoints using depth sensors such as LiDAR or structured light. In modern robotic systems, especially those focused on mapping, it is essential to merge multiple views of the same environment accurately. However, state-of-the-art registration techniques often struggle when large rotational differences exist between point sets or when the data is significantly corrupted by sensor noise. These challenges can lead to misalignments and, consequently, to inaccurate or distorted 3D reconstructions. In this work, we address both these limitations by proposing a robust modification to the classic Iterative Closest Point (ICP) algorithm. Our method, termed Exponential Similarity Matrix ICP (ESM-ICP), integrates a Gaussian-inspired exponential weighting scheme to construct a similarity matrix that dynamically adapts across iterations. This matrix facilitates improved estimation of both rotational and translational components during alignment. We demonstrate the robustness of ESM-ICP in two challenging scenarios: (i) large rotational discrepancies between the source and target point clouds, and (ii) data corrupted by non-Gaussian noise. Our results show that ESM-ICP outperforms traditional geometric registration techniques as well as several recent learning-based methods. To encourage reproducibility and community engagement, our full implementation is made publicly available on GitHub. \url{https://github.com/aralab-unr/ESM_ICP}
\end{abstract}


\section{Introduction}
\label{sec:Intro}
Aligning point clouds/point-sets is a fundamental and active area of research in the field of computer vision and robotics. Also referred to as geometric registration, it involves aligning two 3D point sets representing the same rigid structure or scene, typically captured from different viewpoints using 3D sensors such as LiDAR or structured-light scanners. Applications span diverse domains, including robotics, medical imaging, autonomous driving, and computational chemistry, where accurate alignment of 3D data is crucial. Formally, registration aims to estimate the rigid transformation—comprising rotation and translation—that best aligns a source point cloud with a target point cloud by minimizing a suitable alignment error metric. In this paper, we use the terms: source and the target to denote the point clouds being aligned throughout the discussion. Among the earliest and most widely adopted registration techniques is the Iterative Closest Point (ICP) algorithm~\cite{10.1109/TPAMI.1987.4767965,121791}, which computes correspondences between point sets and refines the transformation iteratively. While ICP performs well when the initial misalignment between the source and the target is small, it is prone to convergence to local minima in the presence of large rotational or translational offsets. In practice, the absence of a good initial estimate and the presence of noise or outliers often degrade ICP's performance. Over the past three decades, numerous variants of ICP have been proposed to address these limitations. These include point-to-plane ICP~\cite{articleptp}, non-linear ICP~\cite{granger2002em,zhou2015affine}, Generalized ICP, and Normal Distributions Transform (NDT)~\cite{ndt1249285}. While these methods offer improvements in specific scenarios, they remain sensitive to large misalignments and noisy data. Alternative techniques such as Go-ICP~\cite{yang2015goicp} and others~\cite{FITZGIBBON20031145,rusinkiewicz2001efficient,si2022review} attempt global optimization over the motion space SE(3), but often incur high computational costs and can still fail under large rotations.

With the advent of deep learning, several learning-based registration methods have emerged, including Deep Closest Point (DCP)~\cite{wang2019deep}, PointNetLK~\cite{aoki2019pointnetlk}, PRNet~\cite{wang2019prnet}, and DeepGMR~\cite{jian2005robust}. While these methods demonstrate promising results on benchmark datasets, they often exhibit limited generalizability and robustness under extreme transformations or noise, particularly when initial alignment is poor or the data is corrupted.

In this paper, we propose a novel registration framework called Exponential Similarity Matrix ICP (ESM-ICP), designed to address the shortcomings of both traditional and learning-based approaches. Building upon the work ~\cite{Ashu_MCCEKFSec, asingandhupej_2018, asingandhupec_2018, SehgalISVC2019, AshuIRC2018}, ESM-ICP introduces a similarity matrix that captures weighted correspondences between the source and target point clouds using a Gaussian-based exponential kernel. This matrix is sparse and updated iteratively during registration, enhancing robustness to large transformations and outliers. 
Key features of our method include: $(1)$ The use of a similarity matrix that encodes the confidence of point correspondences, updated at every iteration. $(2)$ Robust performance under large rotation and translation ranges, even when the point sets are heavily misaligned or corrupted by noise. $(3)$ Superior alignment quality compared to state-of-the-art geometric and learning-based methods. We evaluate ESM-ICP on both clean and noisy datasets, including randomized rigid transformations and non-Gaussian noise injections into one of the 3D point sets. Qualitative results are visualized in Fig.~\ref{fig:ESM_ICP_results}, where the red colored point set denotes the source and the green colored point set denotes the target. In Section \ref{Section_4}, we perform a quantitative evaluation of our approach against several baselines and advanced methods.

Our contributions are summarized as follows:  We introduce ESM-ICP, a robust registration method that incorporates a Gaussian-inspired similarity matrix within the ICP framework. We evaluate ESM-ICP under large transformation scenarios, using random rotations sampled from $[-\pi, \pi]$ and translations across a wide range. We demonstrate improved accuracy over competing methods across multiple datasets and scenarios. We evaluate performance in the presence of outliers and non-Gaussian noise, where our method consistently outperforms traditional and deep learning-based approaches. We publicly release our implementation source codes for reproducibility and community benchmarking.

\begin{figure}[!t]
    \centering
    \begin{tabular}{ccc}
        \small  Input & \small  Input & \small Input \\
        \includegraphics[width=0.28\columnwidth]{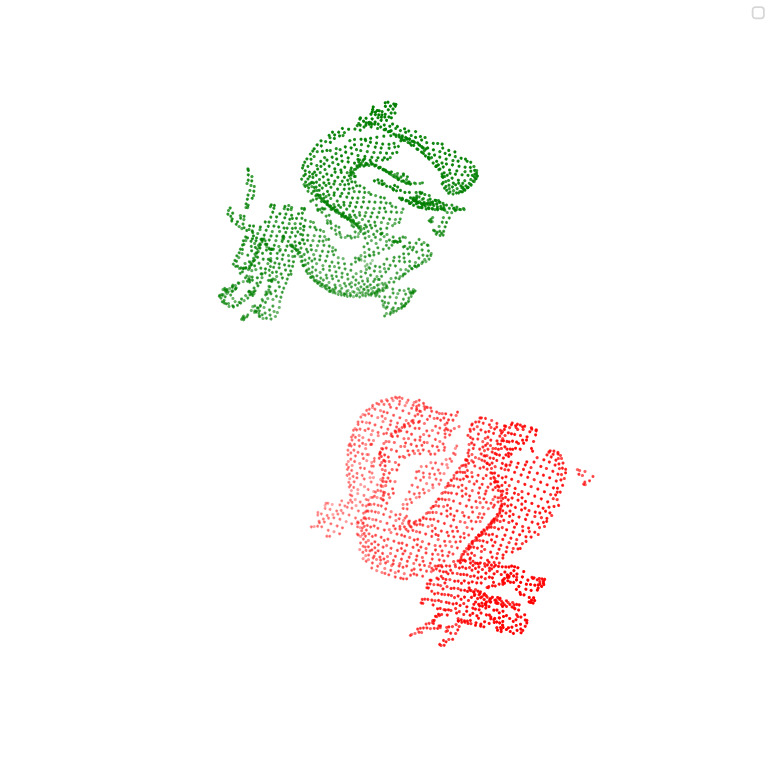} &
        \includegraphics[width=0.28\columnwidth]{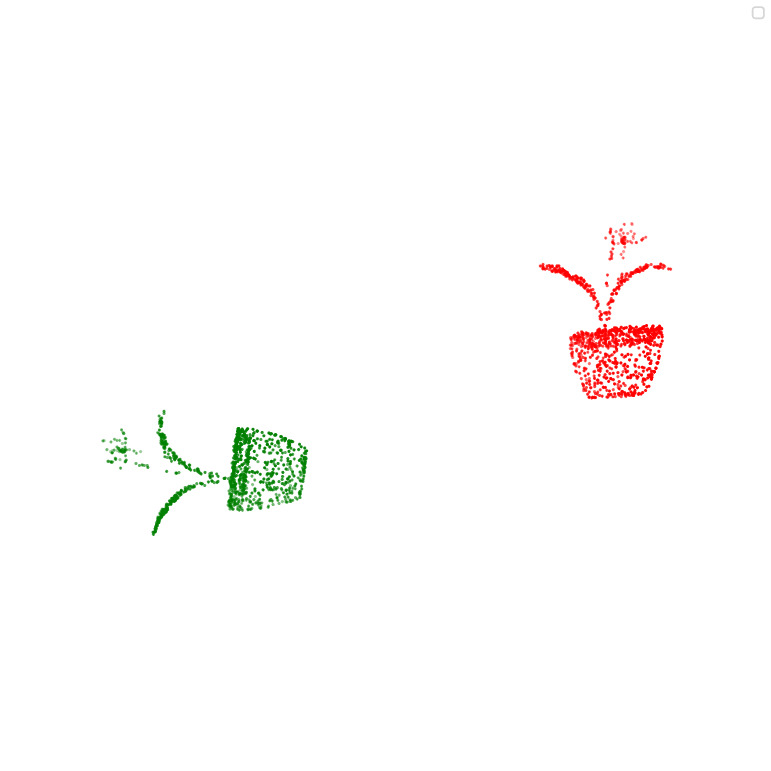} &
        \includegraphics[width=0.28\columnwidth]{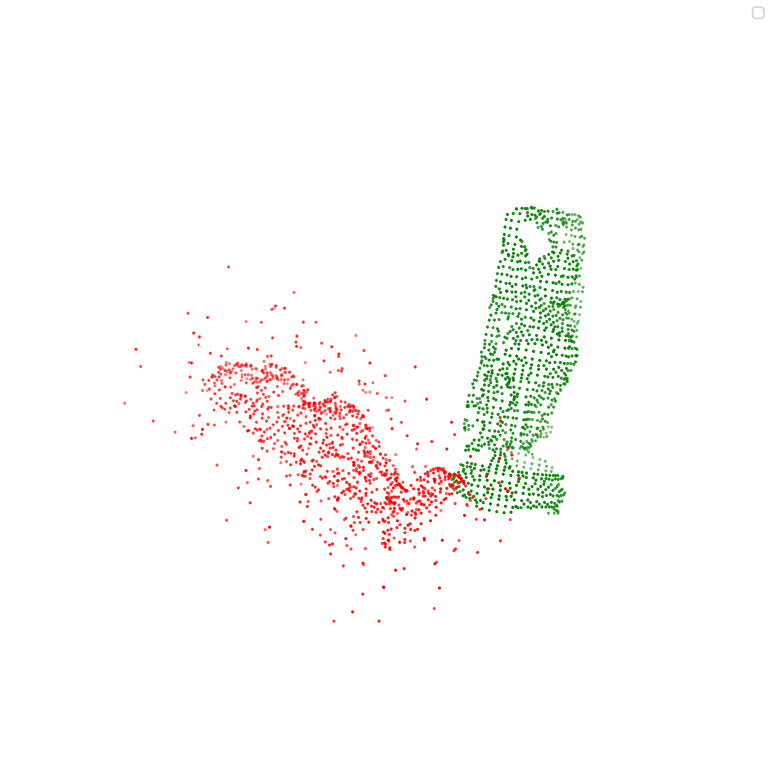} \\
       
        \includegraphics[width=0.28\columnwidth]{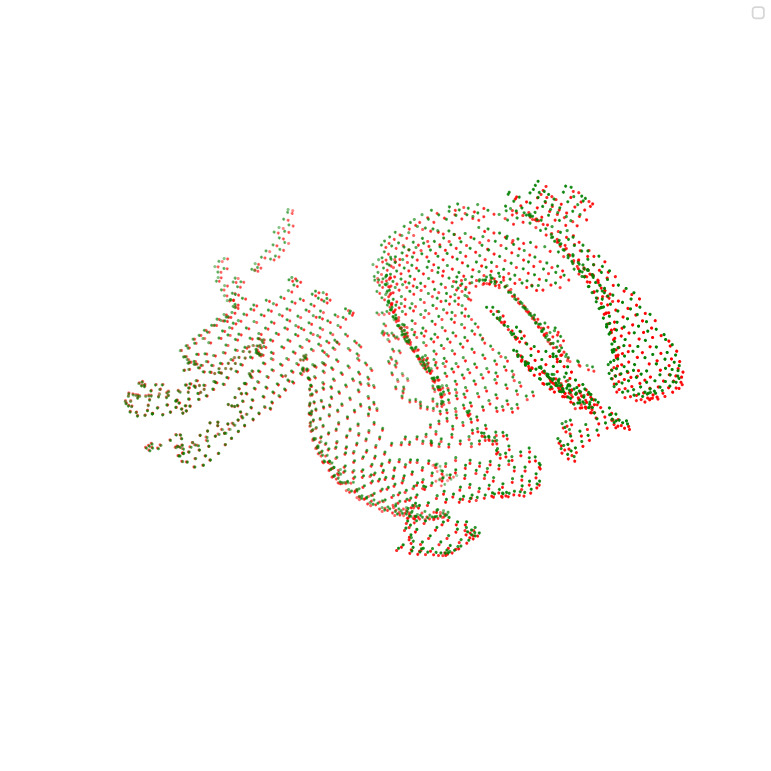} &
        \includegraphics[width=0.28\columnwidth]{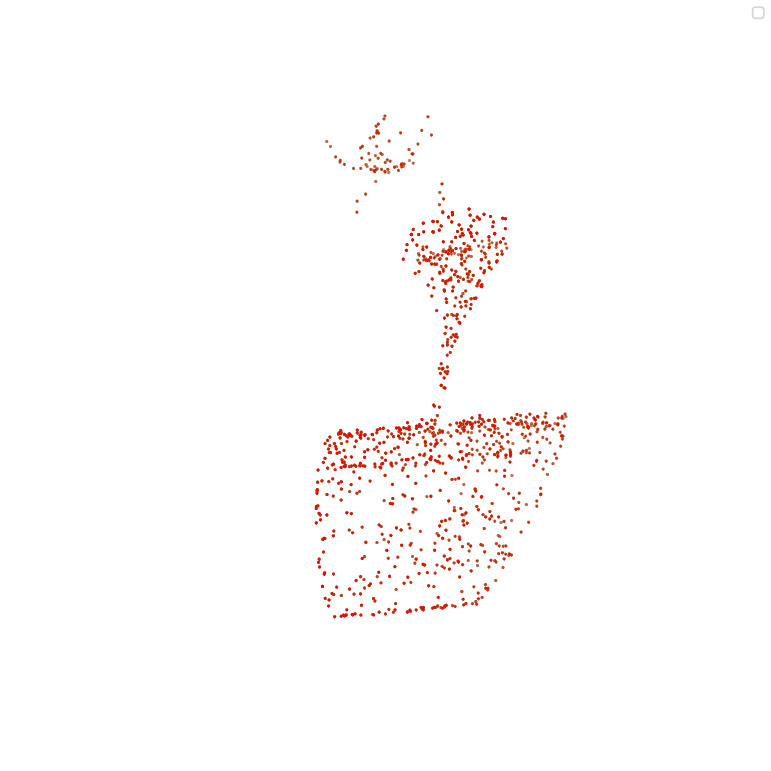} &
        \includegraphics[width=0.28\columnwidth]{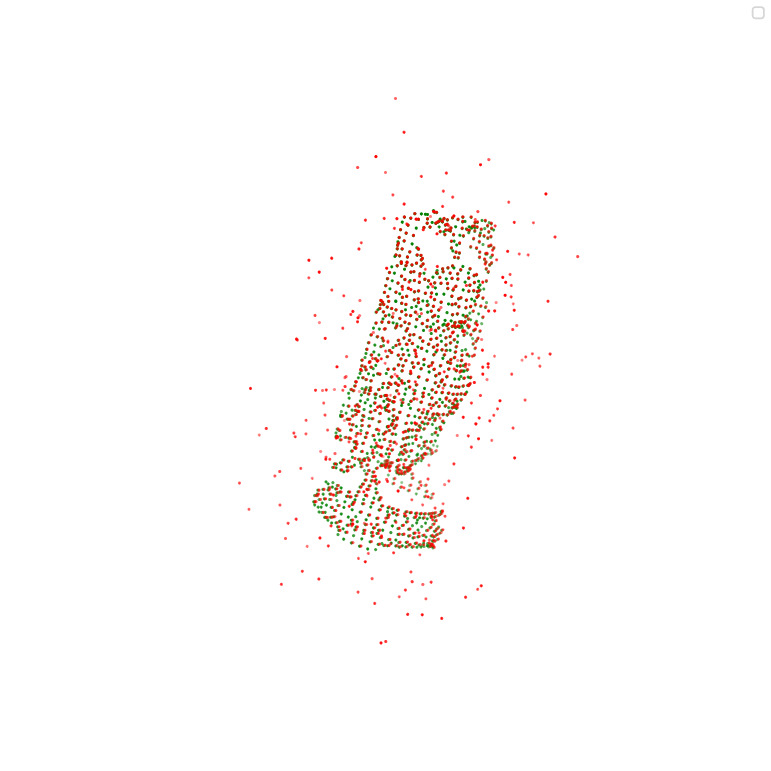} \\
        \small ESM-ICP & \small ESM-ICP & \small ESM-ICP
    \end{tabular}
    \caption{Registration results using ESM-ICP. Top row: input source and target pairs. Bottom row: aligned outputs after applying ESM-ICP.}
    \label{fig:ESM_ICP_results}
    \vspace{-0.5em}
\end{figure}

\section{Related Work}
\label{Background}
The earliest solution to the point cloud registration problem is the classic Iterative Closest Point (ICP) algorithm~\cite{10.1109/TPAMI.1987.4767965,121791,articlewu,zhang2024deep}. ICP iteratively matches corresponding points between two 3D point sets and solves a least-squares problem to refine the transformation. The rotation is computed using Singular Value Decomposition (SVD), while the translation is determined by the difference between the centroids of the point sets. However, traditional ICP struggles with large initial misalignments, is highly sensitive to the initial pose, and is prone to getting trapped in local minima~\cite{si2022review}.

To overcome these limitations, numerous variants of ICP have been developed~\cite {articlewu,huang2021comprehensive}. For instance, Point-to-Plane ICP~\cite{articleptp} incorporates surface normal information from the target point cloud, which accelerates convergence and improves accuracy on smooth surfaces by minimizing the error along the target normal direction. While more efficient than classic ICP, it remains sensitive to large rotations and noise. Generalized ICP (GICP)~\cite{segal2009generalized} addresses this by modeling each point as a Gaussian distribution and using local surface covariances to minimize the Mahalanobis distance between point pairs. This enables alignment based on matching ellipsoidal uncertainty regions rather than simple points or planes.

ICP has also been formulated as an optimization problem, solvable via algorithms like the Levenberg–Marquardt method~\cite{rusu20113d}. Additional robust variants employ outlier trimming techniques to enhance performance in the presence of noise. Beyond ICP-style algorithms, branch-and-bound (BnB) methods like GO-ICP~\cite{yang2015goicp} explore the full SE(3) motion space for global optimization. Although more computationally intensive, these methods offer improved robustness and a higher likelihood of finding the global optimum.

Beyond ICP, probabilistic methods such as the Normalized Distribution Transform (NDT)~\cite{ndt1249285} have been proposed. NDT models local point distributions within voxel grids, eliminating the need for explicit point correspondences and often achieving faster and more accurate registration - though it, too, struggles with large rotations. Another category of approaches reduces registration complexity using four-point congruent sets, which significantly decrease the spatial search space and improve robustness to noise and outliers ~\cite {kuglin1975phase,zitova2003image,guizar2008efficient}. 

Feature-based approaches like RANSAC are also widely adopted, often inspired by 2D image registration techniques like SIFT~\cite{lowe2004distinctive,bay2006surf}. In 3D, descriptors such as Fast Point Feature Histograms (FPFH)~\cite{rusu2009fast} are used to extract keypoints and estimate transformations. Fast Global Registration (FGR) ~\cite{zhou2016fast} is one such example. Similarly, Sample Consensus Initial Alignment (SAC-IA)~\cite{rusu2009sac} uses feature correspondences for coarse alignment, followed by point-to-plane ICP refinement. These methods rely heavily on the quality of the feature descriptors used.

With the rise of deep learning, numerous AI-driven registration techniques have emerged. PointNet~\cite{qi2017pointnet} was the first neural network to directly process unordered point clouds. While widely used in classification, detection, and segmentation tasks, it also influences registration. PointNet captures global shape features but lacks local geometric context. This issue is addressed by PointNet++~\cite{qi2017pointnet++}, which aggregates local features using K-nearest neighbor clusters. Another variant, PointNetLK~\cite{aoki2019pointnetlk}, integrates the Lucas–Kanade algorithm into PointNet’s learned feature space for alignment. More advanced models like DGCNN (Dynamic Graph CNN)~\cite{wang2019dgcnn} dynamically construct graphs in both spatial and feature space to extract richer representations.

Deep Closest Point (DCP)~\cite{wang2019deep} builds upon PointNet by extracting meaningful point-wise feature descriptors and uses attention mechanisms to establish soft correspondences between source and target. Once correspondences are formed, it estimates the transformation using a differentiable SVD layer. While DCP~\cite{wang2019deep} is robust to noise and outliers, it often requires supervised data and pre-trained models.

A related approach, RPM-Net (Robust Point Matching using Learned Features)~\cite{yew2020rpm}, also employs soft correspondence matching through a differentiable Sinkhorn layer with annealing. This improves resilience to noise and partial visibility. It further enhances performance by using a secondary network to predict the optimal annealing parameters during registration.

Despite the variety of available registration algorithms, point cloud alignment can still fail, particularly under conditions of large rotational transformations between source and target point sets. In light of this challenge, we propose a novel variant of the ICP algorithm that performs weighted matching using our proposed similarity matrix. We compare our method to both classical numerical approaches and deep learning-based techniques. Through experimental results, we demonstrate how our approach can effectively resolve these registration challenges.

The remainder of this paper is structured as follows: Section~\ref{Section_2} introduces the background of ICP using SVD and its associated problems. Section~\ref{Section_3} presents our proposed method and its implementation. Section~\ref{Section_4} provides experimental results and evaluation. Section~\ref{Section_5} concludes the paper and outlines future directions.

\section{Background}
\label{Section_2} 
In this section, we discuss the basic ICP problem and highlight the key issues associated with it. As mentioned earlier, for rigid body alignment, we have two 3D point sets: one model (\textbf{\textit{Source}}) and one template (\textbf{\textit{Target}}). We denote the \textbf{\textit{Source}} as $\mathbf{S} = \{ \mathbf{s}_1, \mathbf{s}_2, \mathbf{s}_3, \dots, \mathbf{s}_N \} \subset \mathbb{R}^3$, and the \textbf{\textit{Target}} as $\mathbf{T} = \{ \mathbf{t}_1, \mathbf{t}_2, \mathbf{t}_3, \dots, \mathbf{t}_N \} \subset \mathbb{R}^3$, separated by a random rigid transformation consisting of a rotation matrix $\mathbf{R} \in \mathrm{SO}(3)$ and a translation vector $\boldsymbol{\tau} \in \mathbb{R}^3$. For simplicity, we assume that the \textbf{\textit{Source}} and \textbf{\textit{Target}} point sets contain the same number of 3D points. The vectors $\mathbf{s}_i$ and $\mathbf{t}_i$ represent the 3D coordinates of the Source and Target point clouds, respectively, and are indexed by $i = 1, 2, \dots, N$. For the remainder of this paper, we denote matrices by bold capital letters, vectors by bold lowercase letters, and scalars by regular (non-bold) fonts.

For the rigid body alignment problem, we intend to minimize the mean-squared error $E(\mathbf{R,\boldsymbol{\tau}})$, which can be written as:
\begin{equation}
     E(\mathbf{R},\boldsymbol{\tau}) = \frac{1}{N}\sum_{i}^{N}\|\mathbf{R}\mathbf{s}_{i} + \boldsymbol{\tau} - \mathbf{t}_{i} \|^2.
    \label{ErrorMinimization}
\end{equation} 

To achieve this minimization, we first compute the centroid of both the source and the target point sets as:
\begin{equation}
    \bar{\mathbf{s}} = \frac{1}{N} \sum_{i=1}^{N} \mathbf{s}_i \quad \text{and} \quad \bar{\mathbf{t}} = \frac{1}{N} \sum_{i=1}^{N} \mathbf{t}_i.
    \label{CentroidComputation}
\end{equation}
Then we compute the cross-covariance matrix $\mathbf{H}$ given as:
\begin{equation}
    \mathbf{H} = \sum_{i=1}^{N}(\mathbf{s}_i-\bar{\mathbf{s}})(\mathbf{t}_i-\bar{\mathbf{t}}).
    \label{CrossCovariance_Original}
\end{equation}
Performing SVD on this covariance matrix $\mathbf{H}$ can yield the rotation component. So, if $\mathbf{H} = \mathbf{UDV}^{T}$, then the rotation and the translation component are given by,
\begin{equation}
    \mathbf{R} = \mathbf{VU}^T \quad \text{and} \ \ \ \ \boldsymbol{\tau} = -\mathbf{R}\bar{\mathbf{s}} + \bar{\mathbf{t}}.
    \label{Randtcomputation_Original}
\end{equation}
It can be interpreted that $\mathbf{V},\mathbf{U} \in \mathrm{SO}(3)$. $\mathbf{D}$ is diagonal matrix. Here, the Procrustes problem assumes that all point sets are matched from Source to Target point sets, that is, $\mathbf{s}_i$ to $\mathbf{t}_i$ after performing the final alignment for all $i$. If the correspondences are unknown, then the objective function from Equ. \ref{ErrorMinimization} can be rewritten as:
\begin{equation}
     E(\mathbf{R},\boldsymbol{\tau}) = \frac{1}{N}\sum_{i}^{N}\|\mathbf{R}\mathbf{s}_{i} + \boldsymbol{\tau} - \mathbf{t}_{m(s_i)} \|^2,
    \label{ErrorMinimization_2}
\end{equation} 
where $m$ denotes the mapping from each point in $\mathbf{S}$ to its corresponding points in $\mathbf{T}$. However, it can be easily seen that if we know the optimal rigid transformation, then the mapping $m$ can be recovered, and conversely, given an optimal mapping $m$ transformation can be computed using Equ. \ref{Randtcomputation_Original}. It's a well known \textit{chicken-egg} problem. The algorithm terminates when a certain criterion or threshold is reached, however, is extremely prone to local optima. If the initial alignment is far, it may get stuck, yielding a poor estimate $m$. We try to address this problem in a fairly intuitive way, where we pick only a few points based on a given criterion in each iteration, and we see how it can achieve global registration as the iteration increases.

\section{Weighted Cross-Covariance Matrix}
\label{Section_3} 
As it is clear from the standard ICP framework that the cross-covariance matrix, referenced in Equation~\ref{CrossCovariance_Original} - is computed using pairs of points in the source and target point sets after centering both with respect to their respective centroids. This conventional approach assumes a hard correspondence: each source point is matched to a single closest target point, and all point pairs contribute equally to the alignment.

In contrast, our method introduces a significant generalization by computing a weighted cross-covariance matrix, which is iteratively refined as the algorithm progresses. Instead of assuming binary correspondences, we adopt a soft correspondence model that evaluates the similarity between each point in the source set and its paired point in the target set using a Gaussian kernel. This results in a similarity matrix that modulates the influence of each correspondence based on geometric proximity.

Specifically, we define a Gaussian kernel function as:
\begin{equation}
f(x) = \exp\left(-\frac{x^2}{2\sigma^2}\right),
\label{Gaussian_kernel}
\end{equation}
where $x$ is the Euclidean distance between a pair of corresponding points, and $\sigma > 0$ is a user-defined scale parameter that controls the sensitivity of the similarity measure. This kernel ensures that correspondences with smaller distances are weighted more heavily, while those with larger discrepancies are exponentially suppressed. As a result, outliers and poor matches contribute minimally to the alignment process, increasing the robustness of the algorithm in the presence of noise, occlusions, or non-rigid deformations.

This weighted formulation not only generalizes the classical ICP's assumption of equal-weighted correspondences but also provides a smooth and differentiable measure of alignment quality, which is particularly advantageous when dealing with ambiguous or noisy data. The weighted cross-covariance matrix thus encapsulates both geometric structure and probabilistic confidence in the pairwise alignments, leading to more accurate and stable transformation estimates.

From the Gaussian Kernel function in Equ. \ref{Gaussian_kernel}, we define our alignment problem as a weighted Procrustes problem, where the weights for each corresponding pair from the source to the target point sets are defined as:

\begin{equation}
w_i = \exp\left(-\frac{\| \mathbf{s}_i - \mathbf{t}_{c(i)} \|^2}{2\sigma^2}\right),
\label{weights}
\end{equation}
where $c(i)$ is the index of the corresponding points from each point in the source point set to the target point set. These correspondences, as mentioned, are the closest points based on Euclidean distance.  Later, we construct a symmetric similarity matrix \( \mathbf{M} \) of size \( N \times N \) such that:
\begin{equation}
\mathbf{M}(i, c(i)) = \mathbf{M}(c(i), i) = w_i.
\label{M_matrix_construction}    
\end{equation}

One may note that our similarity measurement matrix $\mathbf{M}$ is symmetric, i.e., $\mathbf{M} = \mathbf{M}^\top$, where $\top$ denotes the matrix transpose.
Based on the above definition, we compute the $\mathbf{H}$ matrix as:
\begin{equation}
    \mathbf{H} = \sum_{i=1}^{N}(\mathbf{s}_i-\bar{\mathbf{s}})\mathbf{M}(\mathbf{t}_i-\bar{\mathbf{t}}).
    \label{CrossCovariance_ESM}
\end{equation}
Performing SVD as equivalent to the traditional ICP on this covariance matrix $\mathbf{H}$ can yield the rotation component. So, if $\mathbf{H} = \mathbf{UDV}^{T}$ , then the rotation and the translation component is given by,
\begin{equation}
    \mathbf{R} = \mathbf{VU}^T \quad \text{and} \ \ \ \ \boldsymbol{\tau} = -\mathbf{R}\bar{\mathbf{s}} + \bar{\mathbf{t}}.
    \label{Randtcomputation_ESM}
\end{equation}
Based on the above formulation, we arrive at minimizing the function:
\begin{equation}
\min_{\mathbf{R} \in \mathrm{SO}(3)} \sum_{i=1}^{N} w_i \left\| \mathbf{R} \mathbf{s}_i' - \mathbf{t}_i' \right\|^2,
\end{equation}
and the weighted alignment error at iteration k is:
\begin{equation}
E_k = \sum_{i=1}^{N} w_i^{(k)} \left\| \mathbf{R}_k \mathbf{s}_i + \boldsymbol{\tau}_k - \mathbf{t}_i \right\|^2.
\end{equation}
Note that the similarity weights $w_i$ are bounded ($0<w_i<1$). After every iteration, the computed transformation $\mathbf{A}_k$ is the minimizer of $E_k$ so $E_{k+1} \leq E_k$. Here $\mathbf{A}_k$ can be written as :
\begin{equation}
\mathbf{A}_k =
\begin{bmatrix}
\mathbf{R}_k & \boldsymbol{\tau}_k \\
\mathbf{0}^\top & 1 
\end{bmatrix}
.
\end{equation}

\subsection{Convergence Analysis of ESM-ICP}

\begin{theorem}
Let $\mathbf{A}_k$ be the sequence of rigid transformations computed by ESM-ICP over iterations. If the weights $w_i$ are bounded and continuous (it's guaranteed since $0<w_i<1$), and the similarity matrix remains well-conditioned, then the sequence converges $\mathbf{A}_k$ to a fixed transformation $\mathbf{A}^*$.
\end{theorem}

\begin{proof}
At each iteration, the objective function minimized is:
\begin{equation}
E_k = \sum_{i=1}^{N} w_i^{(k)} \left\| \mathbf{R}_k \mathbf{s}_i + \boldsymbol{\tau}_k - \mathbf{t}_i \right\|^2.
\end{equation}

This is a weighted least-squares error where weights $w_i^{(k)} \in (0,1]$  are derived from the Gaussian kernel. As the transformation is refined, the distances $\left\| \mathbf{s}_i - \mathbf{t}_{c(i)} \right\|$ decrease, stabilizing the weights. The updated transformation $\mathbf{A}_{k+1}$ minimizes , leading to:
\begin{equation}
E_{k+1} < E_k.
\end{equation}

The sequence $E_k$ is thus monotonically decreasing and bounded below, implying convergence. Furthermore, the transformation updates satisfy $\left\| \mathbf{A}_{k+1} - \mathbf{A}_{k} \right\|$, yielding convergence to a fixed point $\mathbf{A}^*$.
\end{proof}

\begin{corollary}
Let $\mathbf{M}$ be the similarity matrix used in ESM-ICP. After $k$ iterations, if the source and target point sets are sufficiently close such that the computed weights satisfy $w_i \to 1$, the ESM-ICP formulation effectively reduces to the standard ICP algorithm.
\end{corollary}

\begin{proof}
Suppose that after $k$ iterations, the weights $w_i$ approach $1$ for all points, as governed by the exponential weighting function. This implies the emergence of near-perfect one-to-one correspondences between the source and target point sets. Consequently, according to Equ.~\ref{M_matrix_construction}, the similarity matrix $\mathbf{M}$ asymptotically becomes a diagonal matrix with $w_i \to 1$, effectively reducing $\mathbf{M}$ to the identity matrix.

Substituting $\mathbf{M} = \mathbf{I}$ into Equ.~\ref{CrossCovariance_ESM} yields the original cross-covariance formulation in Equ.~\ref{CrossCovariance_Original}, which corresponds to the standard ICP algorithm.
\end{proof}

Our proposed approach is algorithmically described in Algorithm~\ref{CoSM_algo}.
\algrenewcommand\alglinenumber[1]{\footnotesize#1:}
\begin{algorithm}
\small  
\caption{ESM-ICP Algorithm}
\label{CoSM_algo}
\begin{algorithmic}[1]
\Function{\textbf{ReadDataSets}}{$\mathbf{S}, \mathbf{T}$}
\While {not converged}
    \State $\mathbf{c} \gets \textbf{ComputeCorrespondence}(\mathbf{S}, \mathbf{T})$
    \State $\mathbf{M} \gets \text{zeros}(N, N)$.
    \For{$i = 1$ to $N$}
        \State $\mathbf{d} \gets \| \mathbf{s}_i - \mathbf{t}_{c(i)} \|$.
        \State $w_i \gets \exp\left(-\frac{ \| \mathbf{s}_i - \mathbf{t}_{c(i)} \|^2 }{2\sigma^2} \right)$.
        \State $\mathbf{M}(i, c(i)) = w_i$, \quad $\mathbf{M}(c(i), i) = w_i$.
    \EndFor
    \State $\bar{\mathbf{s}} \gets \frac{1}{N} \sum \mathbf{s}_i$, \quad $\bar{\mathbf{t}} \gets \frac{1}{N} \sum \mathbf{t}_i$.
    \State $\mathbf{H} \gets \sum (\mathbf{s}_i - \bar{\mathbf{s}})\mathbf{M}(\mathbf{t}_i - \bar{\mathbf{t}})$.
    \State SVD: $\mathbf{H} = \mathbf{UDV}^\top$.
    \State $\mathbf{R} \gets \mathbf{V} \mathbf{U}^\top$.
    \State $\boldsymbol{\tau} \gets -\mathbf{R} \bar{\mathbf{s}} + \bar{\mathbf{t}}$.
\EndWhile
\EndFunction
\end{algorithmic}
\end{algorithm}

\section{Experiments}\label{Experiments}
\label{Section_4} 
The effectiveness of the proposed ESM-ICP algorithm is assessed using two widely recognized benchmark datasets: the Stanford Bunny \cite{datasetsource1} and ModelNet40 \cite{wu20153d}. While both datasets are utilized, the majority of experiments are conducted on ModelNet40 due to its diversity and scale. Within our experimental pipeline, point cloud data is loaded in \textit{pcd} format. A random rigid transformation is applied to one of the point clouds, which is subsequently designated as the \textit{Source}, while the unaltered point cloud serves as the \textit{Target}. The objective is to accurately estimate the transformation that aligns the Source to the Target using our ESM-ICP method and to quantitatively evaluate the quality of this alignment.
We benchmark the performance of ESM-ICP against several baseline and state-of-the-art registration algorithms, including traditional geometry-based methods such as standard ICP~\cite{10.1109/TPAMI.1987.4767965,121791}, ICP Point-to-Plane~\cite{articleptp}, Go-ICP~\cite{yang2015goicp}, and ICP Non-Linear~\cite{granger2002em,zhou2015affine}, as well as learning-based methods like Deep Closest Point (DCP)~\cite{wang2019deep}, PointNetLK~\cite{aoki2019pointnetlk}, RpmNet~\cite{confcvprYewL20}, and DeepGMR~\cite{jian2005robust}. To assess robustness, we also introduce outliers into the Source point cloud and analyze the algorithm's performance under such perturbations.
Evaluation is primarily conducted using three error metrics—Root Mean Square Error (RMSE), Mean Absolute Error (MAE), and Mean Square Error (MSE)—calculated between the predicted and ground-truth transformations. Successful rigid alignment is characterized by RMSE values approaching zero. All experiments are performed on a system equipped with an Intel i7 CPU (2.70 GHz) and 32 GB of RAM. For deep learning-based approaches, we leverage an NVIDIA GTX 1070 GPU with 8 GB of VRAM.

To explore the internal behavior of our algorithm, we initially focus on the Stanford Bunny dataset. In this setting, we analyze the evolution of the similarity matrix $\mathbf{M}$ over successive iterations to illustrate how $\mathbf{M}$ progressively approaches an identity matrix as the algorithm converges.
\subsection{Evolution of $\mathbf{M}$}
As an initial evaluation, we utilize the well-known Stanford Bunny dataset to analyze the behavior of our algorithm in a controlled setting. To simplify visualization and enhance interpretability, we downsample the point cloud using a voxel grid filter with a leaf size of 0.01, implemented via the Point Cloud Library (PCL v1.15)~\cite{rusu20113d}. The reduced number of points allows for clearer observation of the alignment process between the source and target point sets.

A random rigid transformation with a large rotational component is applied to the source point cloud. Specifically, the rotation is defined by Euler angles $(r, p, y) = (0, 0, 4.53)$, while the translation vector is sampled uniformly in the range $[-0.1, 0.1]$. It is important to note that, due to the centroid-based nature of our transformation computation, the absolute translation offset has minimal influence on convergence. The focus remains on recovering accurate rotational alignment despite significant initial misalignment. Figure~\ref{fig:test1_bun_downsampled}(a) depicts the initial separation between the source and target point sets, while Figures~\ref{fig:test1_bun_downsampled}(b) and (c) demonstrate the convergence of the algorithm over 14 and 15 iterations, respectively.

For all experiments, we range the value of $\sigma = {0.05,1.0}$ in our exponential similarity weighting function depending on the point set density. For example, if the point set is dense (captured from a high fidelity sensor), we can set the value of sigma to be small, and large otherwise. This parameter governs the influence of point-to-point distances on the similarity matrix and is crucial in moderating the robustness of the algorithm to noise and outliers.
\begin{figure}[!htb]
\centering
\subfloat[]{\includegraphics[width=0.15\textwidth]{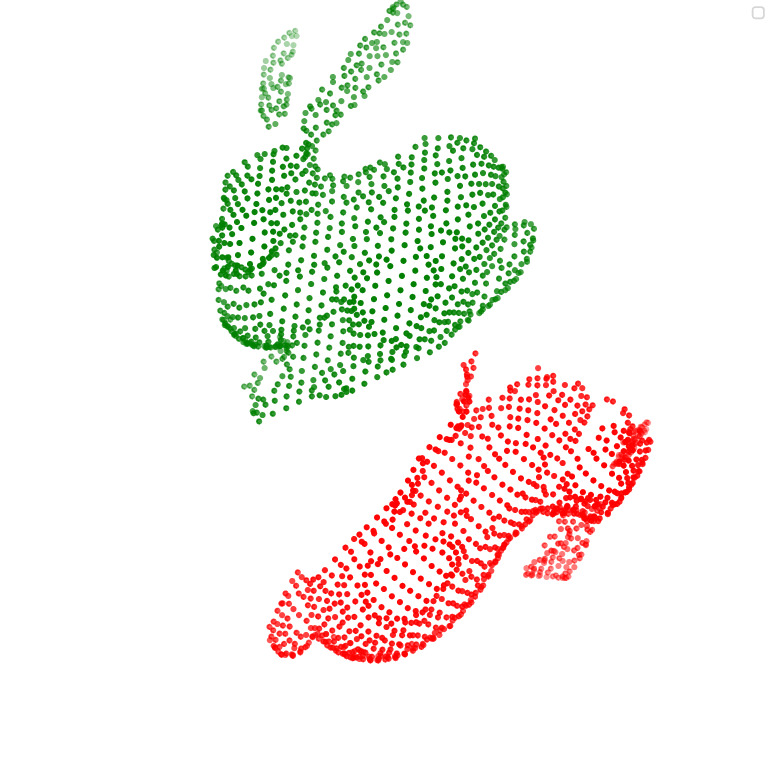}} \hfil
\subfloat[]{\includegraphics[width=0.15\textwidth]{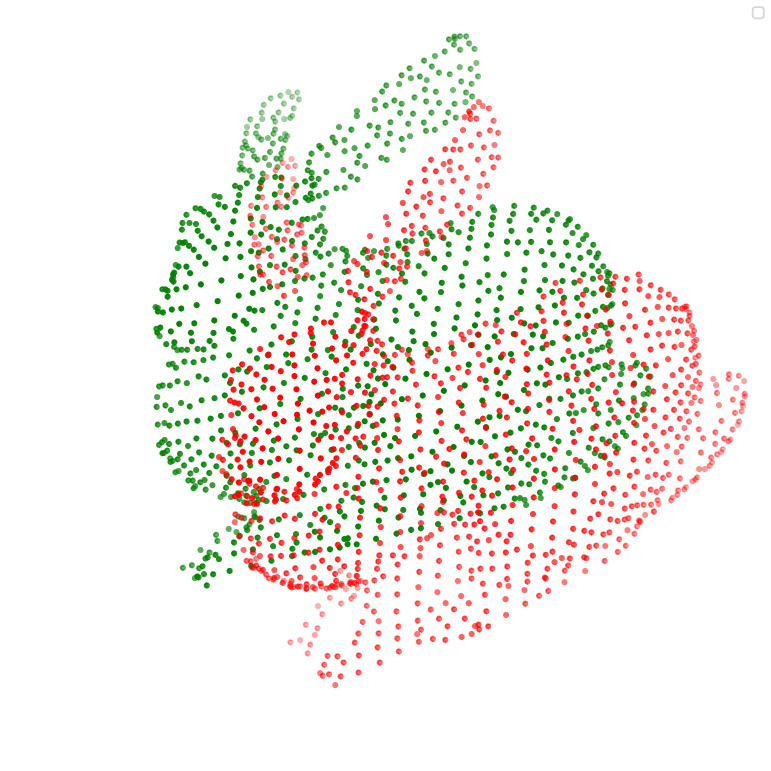}} \hfill
\subfloat[]{\includegraphics[width=0.15\textwidth]{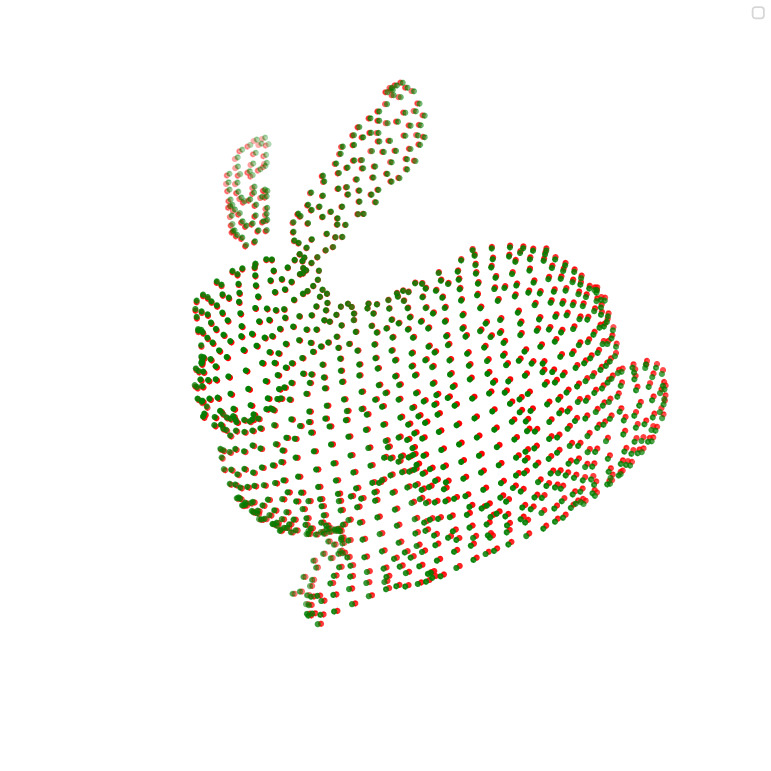}} \hfil
  \caption{(a) shows the 2 bunny rabbit point cloud (Red: Source and Green: Target) down-sampled using voxel grid filtering of leaf size $0.01$. The transformation between them is $(r,p,y)=(0,0,4.53)$ and $({\tau}_x,{\tau}_y,{\tau}_z) = (0,0,0.3)$  (b) shows the alignment after 14 iterations. (c) shows the alignment after 15 iterations.}
  \label{fig:test1_bun_downsampled}
\vspace{-0pt}
\end{figure}
To analyze the internal mechanics of our approach, particularly the evolution of the similarity matrix $\mathbf{M}$, we track its values across iterations. As outlined earlier, $\mathbf{M}$ is reinitialized to a zero matrix at each iteration, and its entries are updated based on the exponential weighting of current point correspondences. Figure~\ref{fig:Heatmap} visualizes how the structure of $\mathbf{M}$ evolves, illustrating a gradual increase in one-to-one correspondences as the alignment progresses.

To make this matrix more interpretable, we conduct the same experiment using a heavily downsampled version of the Bunny dataset with a voxel leaf size of 0.06. This results in just 21 points, generating a $21 \times 21$ similarity matrix that is easier to visualize and analyze. From the heatmap in Fig.~\ref{fig:Heatmap}, it is evident that only correspondences with substantial weights (highlighted in yellow) significantly contribute to the transformation computation through matrix $\mathbf{H}$. The remaining elements (shown in blue) have near-zero values and are effectively ignored.

As iterations proceed, the structure of $\mathbf{M}$ begins to resemble the identity matrix, indicating a strong one-to-one correspondence between source and target points. Once convergence is close enough, the algorithm transitions into a behavior similar to classical ICP~\cite{10.1109/TPAMI.1987.4767965,121791}. Notably, the final RMSE achieved in this test was $0.000837512$ after 15 iterations, highlighting the high precision of the recovered alignment.
\begin{figure*}[!htb]
\centering
\subfloat[]{\includegraphics[width=0.9\textwidth,height=7cm]{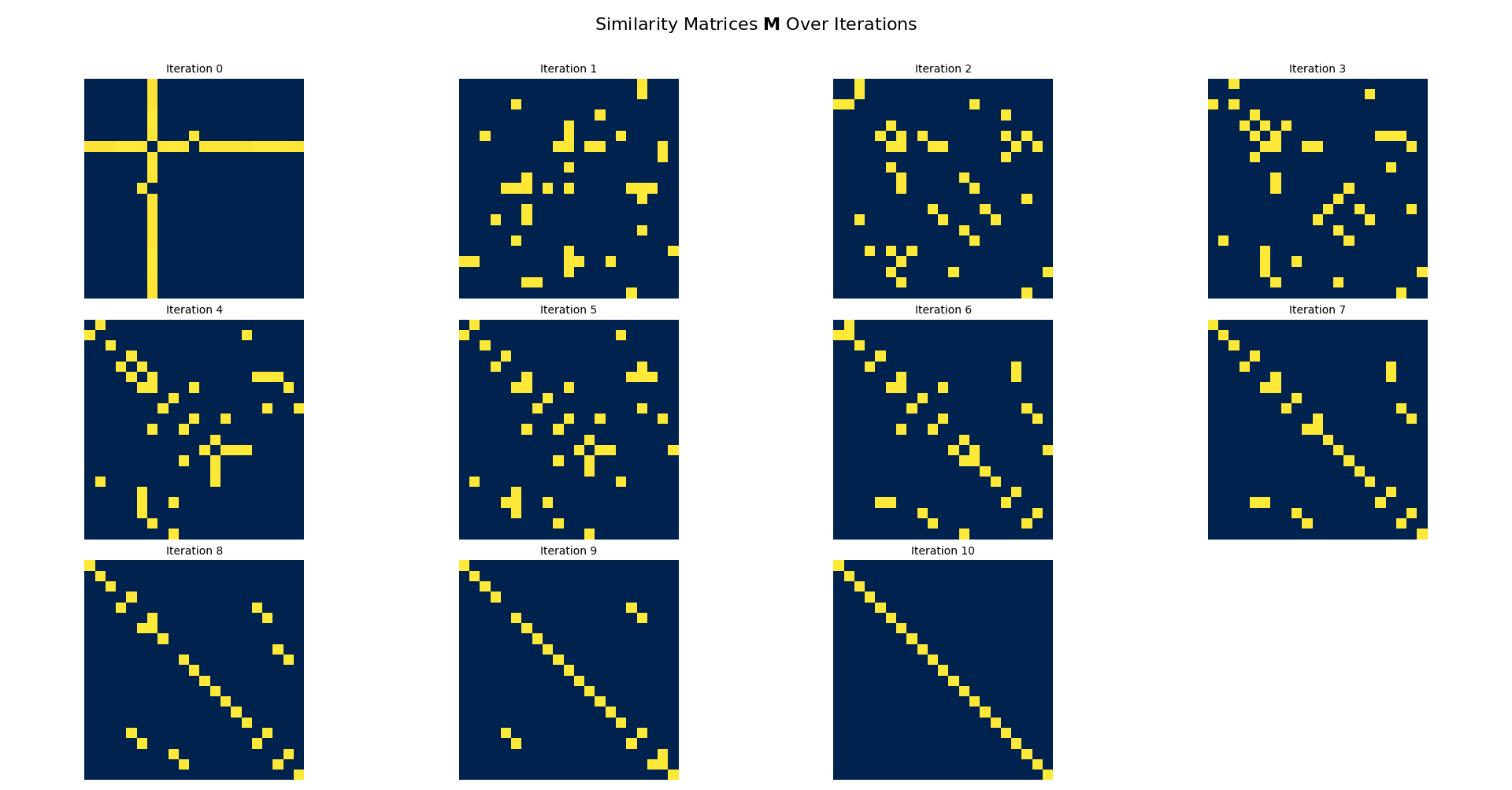}} \hfil
  \caption{(a) Evolution of the similarity matrix $\mathbf{M}$ ($21 \times 21$) across iterations. (Note: the bunny rabbit point cloud dataset was sampled with a leaf size of $0.06$, reducing the matrix size for better visualization purposes.) Yellow entries denote high-weight correspondences; blue entries indicate near-zero weights.}
  \label{fig:Heatmap}
\vspace{-0pt}
\end{figure*}

\subsection{Comparison to other methods}
The comparative analysis is extended to include a range of geometric registration techniques, such as standard Iterative Closest Point (ICP~\cite{10.1109/TPAMI.1987.4767965,121791}), ICP Point-to-Plane~\cite{articleptp}, ICP Non-Linear (ICP-NL)~\cite{granger2002em,zhou2015affine}, and the globally optimal Go-ICP~\cite{yang2015goicp} algorithm. All of these methods are categorized as geometry-based alignment approaches. For evaluation, we use the widely adopted ModelNet40 dataset, applying random rigid transformations to simulate diverse registration challenges. The transformation parameters include rotation angles sampled uniformly from the range $[-\pi, \pi]$ and translations drawn from both $[-1.0, 1.0]$ and a wider interval of $[-100, 100]$ for more extreme tests.

Importantly, no additional pre-processing steps—such as voxel grid filtering, normal estimation, or downsampling—are applied to the ModelNet40 dataset. This decision was made to assess the robustness of the registration algorithms on raw, unfiltered point cloud data. Additionally, we omit runtime comparisons, as the computational performance of our proposed ESM-ICP algorithm is similar to standard ICP~\cite{10.1109/TPAMI.1987.4767965,121791, 10.1109/TPAMI.1987.4767965} and remains suitable for real-time applications.

Figure~\ref{fig:geometric_based_approaches_comparison} provides a visual comparison of ESM-ICP against the aforementioned geometric methods. While Go-ICP~\cite{yang2015goicp} is known for its robustness to large rotations and occasionally demonstrates strong performance, we observed failure cases in certain complex configurations. In contrast, our method consistently succeeds across a wide range of transformation scenarios. Specifically, we evaluated our approach on over 2,000 randomly sampled rotation combinations within $[-\pi, \pi]$ and found that ESM-ICP successfully aligned all test pairs within 100 iterations.

Traditional methods, including standard ICP~\cite{10.1109/TPAMI.1987.4767965,121791}, ICP-NL and ICP Point-to-Plane~\cite{articleptp}, frequently failed in our evaluation, particularly in scenarios involving large rotational misalignments. These failures align with known limitations of such techniques, which often struggle without a good initial alignment guess.

In parallel, we also compare ESM-ICP with several deep learning-based registration approaches. These include Deep Closest Point (DCP~\cite{wang2019deep}), PointNetLK~\cite{aoki2019pointnetlk}, RpmNet~\cite{confcvprYewL20}, and DeepGMR~\cite{jian2005robust}, all evaluated using their publicly available pre-trained models provided by the Learning3D library. Although these models were trained on the ModelNet40 dataset, they typically follow a setup where the testing set differs from the training set. For example, DCP~\cite{wang2019deep} applies random rigid transformations during training, drawing rotation angles from the interval $[0^\circ, 45^\circ]$ and translations from $[-0.5, 0.5]$. However, since our evaluation often involves rotations significantly exceeding $45^\circ$, these models—including RpmNet~\cite{confcvprYewL20} and DeepGMR~\cite{jian2005robust}—exhibited degraded performance and often failed to produce accurate alignments.

As shown in Fig. \ref{fig:ai_method_comparison}, ESM-ICP outperforms all evaluated deep learning models under a wide range of transformation conditions. Quantitative results are presented in Table~\ref{tab:modelnet40_airplane_results} for the Airplane class and in Table~\ref{tab:modelnet40_person_results} for the Person class. Both illustrate the superior accuracy of ESM-ICP.

In contrast to deep learning models, Go-ICP~\cite{yang2015goicp} shows comparatively better performance under large rotations, however, it too can fail in scenarios with complex and coupled axis rotations. Moreover, Go-ICP~\cite{yang2015goicp}, like most methods, struggles when the source point cloud is contaminated by noise.

To further assess robustness, we also evaluate registration performance in the presence of noise. Table~\ref{tab:modelnet40_bottle_outliers_results} and Table~\ref{tab:modelnet40_vase_outliers_results} show results when outliers are introduced in the source point cloud. Noise is injected post-transformation using a mixture of Gaussian distributions: $\mathcal{N}(0, 0.01)$ clipped to $[-0.05, 0.05]$, $\mathcal{N}(0, 0.04)$ clipped to $[-1.0, 1.0]$, and $\mathcal{N}(0, 0.1)$ clipped to $[-10, 10]$, effectively modeling non-Gaussian behavior. In these challenging scenarios, most traditional and learning-based methods fail to recover accurate transformations.

In contrast, our proposed ESM-ICP method demonstrates strong robustness against such distortions. This is attributed to our similarity matrix $\mathbf{M}$, which leverages an exponential Gaussian weighting strategy. This mechanism reduces the influence of outlier correspondences during optimization, thereby enhancing alignment accuracy. The consistent performance of ESM-ICP across both clean and noisy data strongly highlights its effectiveness for real-world registration applications.
\begin{figure*}[!htb]
\centering
\renewcommand{\arraystretch}{0.0} 
\setlength{\tabcolsep}{4pt}      
\begin{tabular}{c c c c c c}
\textbf{Input} & \textbf{ICP}~\cite{10.1109/TPAMI.1987.4767965,121791} & \textbf{ICP-NL}~\cite{granger2002em,zhou2015affine} & \textbf{GOICP}~\cite{yang2015goicp} & \textbf{ICP-P2PL}~\cite{articleptp} & \textbf{ESM-ICP} \\
\includegraphics[width=0.15\textwidth]{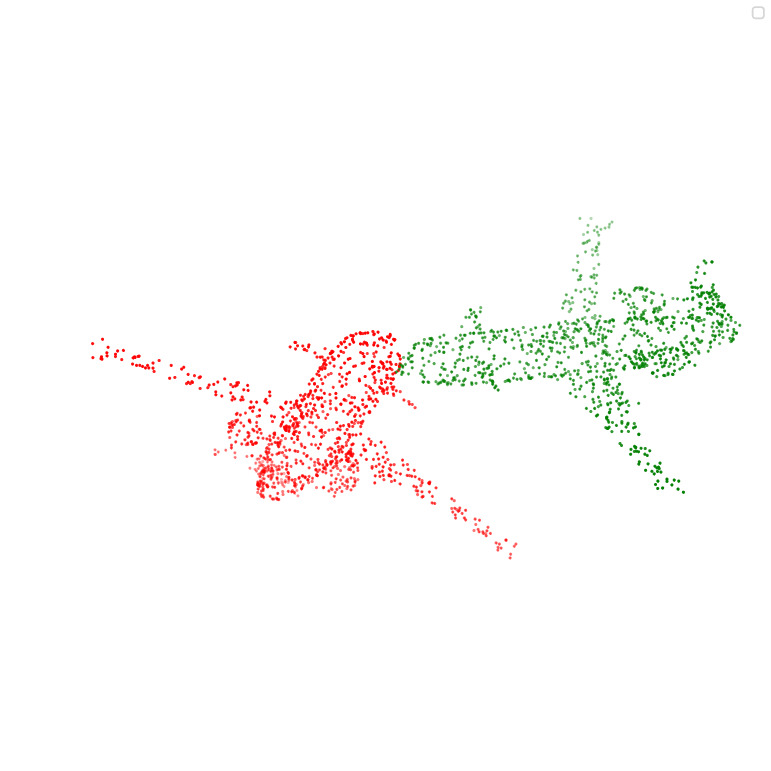} & 
\includegraphics[width=0.15\textwidth]{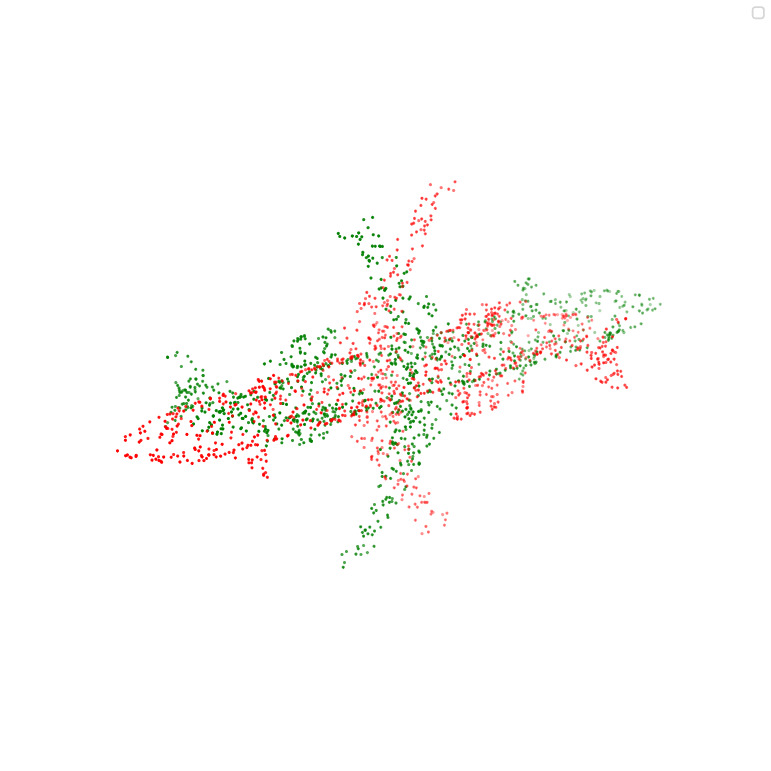} &
\includegraphics[width=0.15\textwidth]{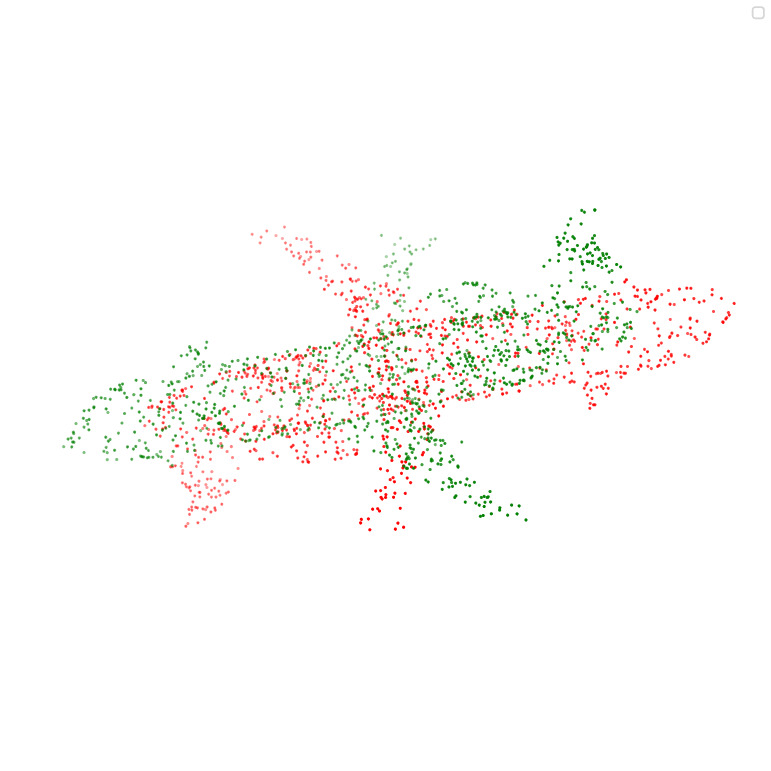} &
\includegraphics[width=0.15\textwidth]{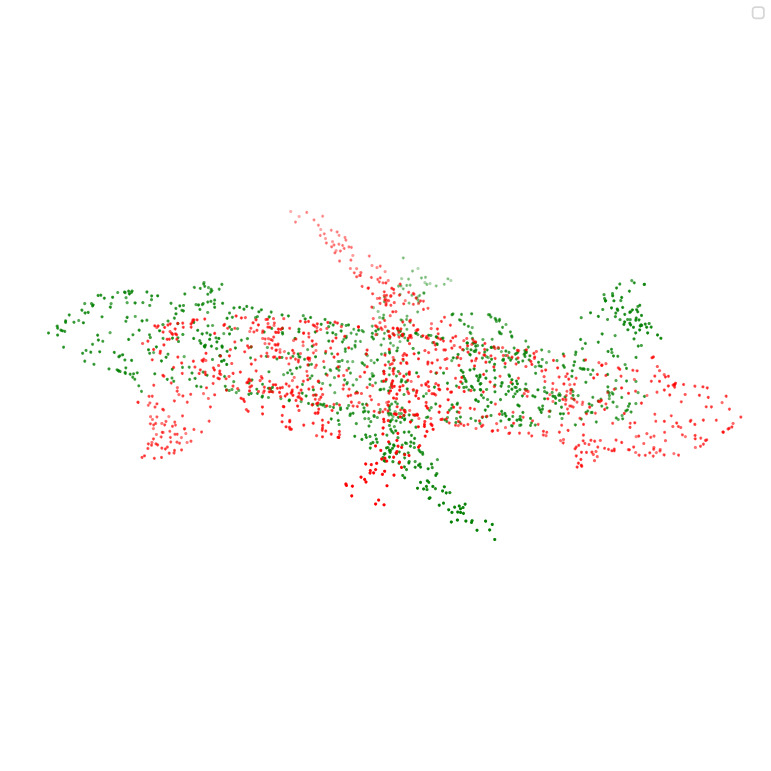} &
\includegraphics[width=0.15\textwidth]{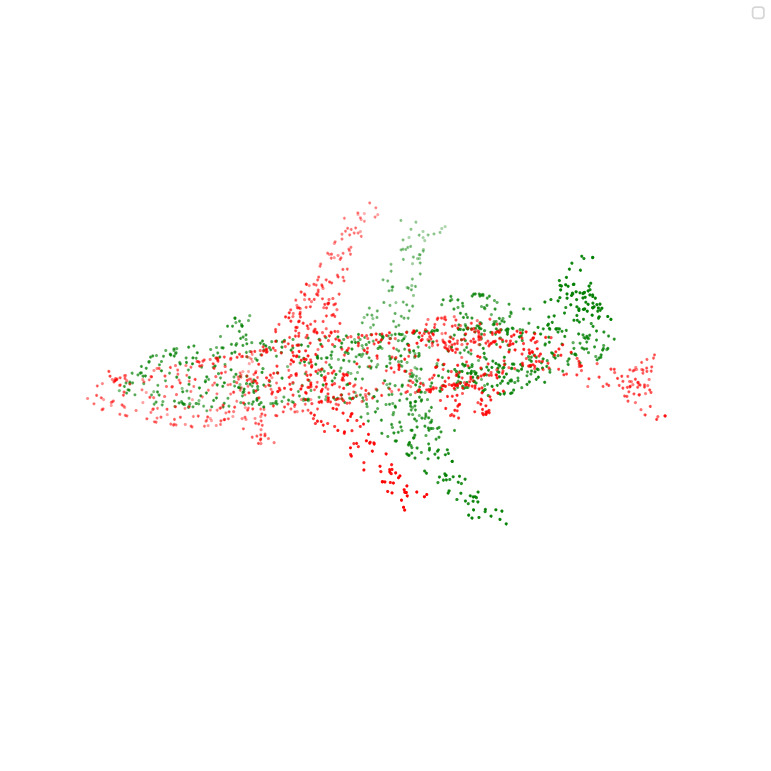} & 
\includegraphics[width=0.15\textwidth]{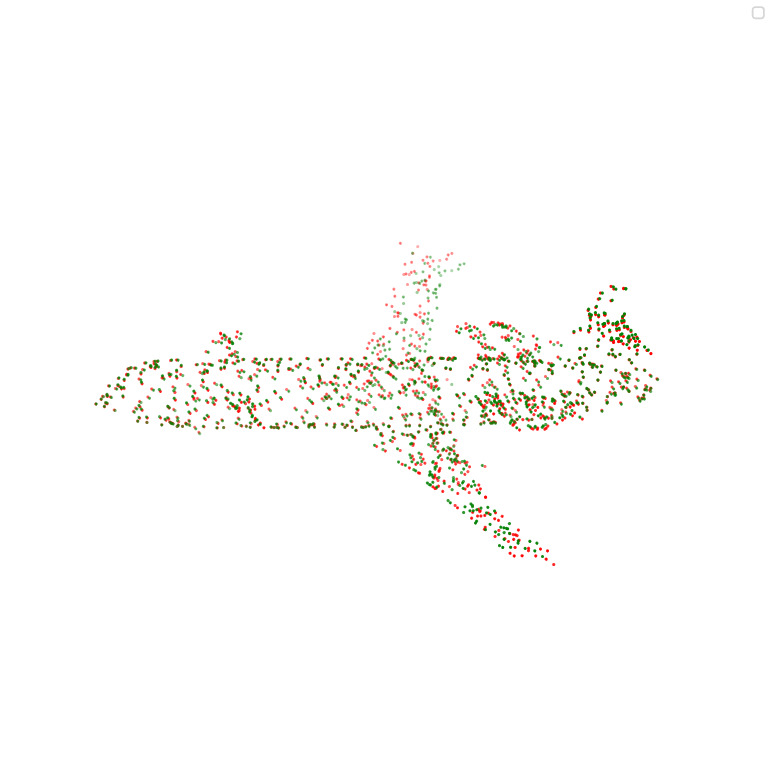} 
\\
\includegraphics[width=0.15\textwidth]{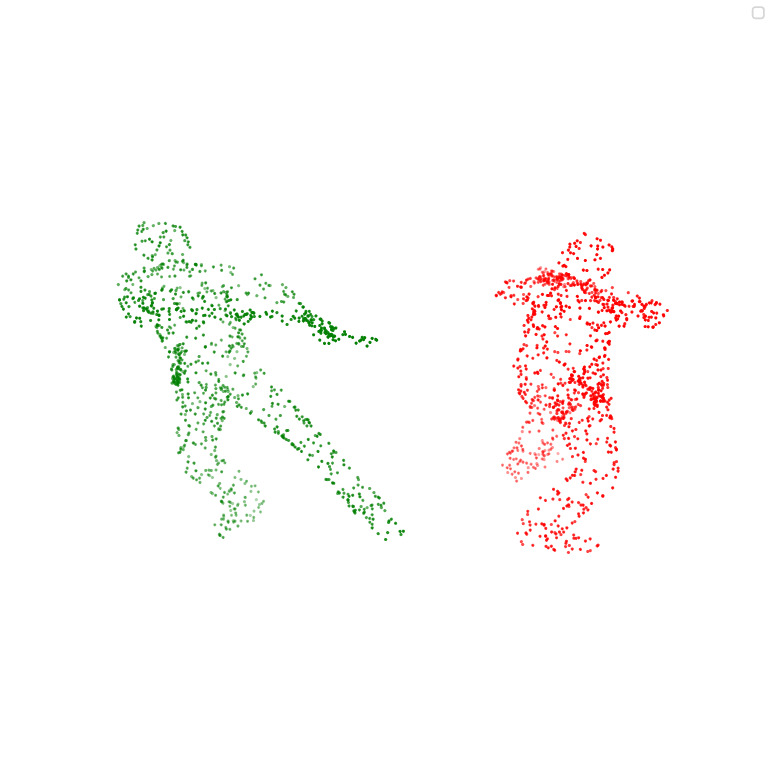} &
\includegraphics[width=0.15\textwidth]{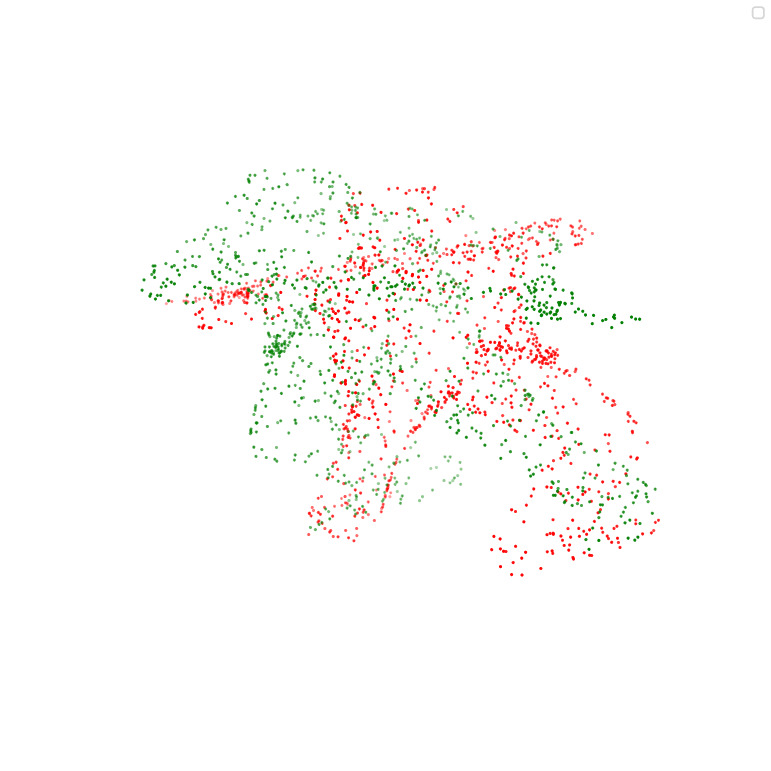} &
\includegraphics[width=0.15\textwidth]{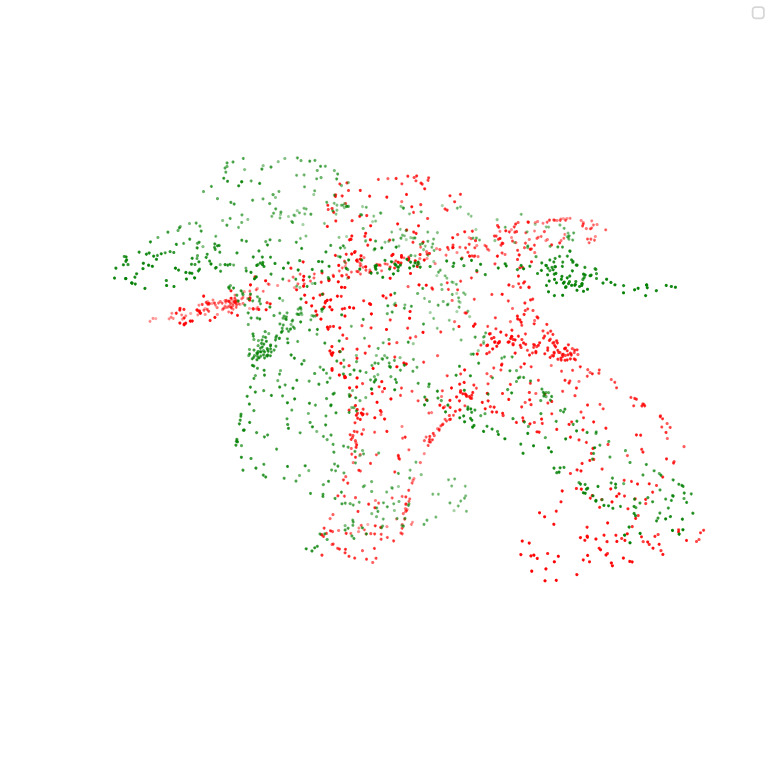} &
\includegraphics[width=0.15\textwidth]{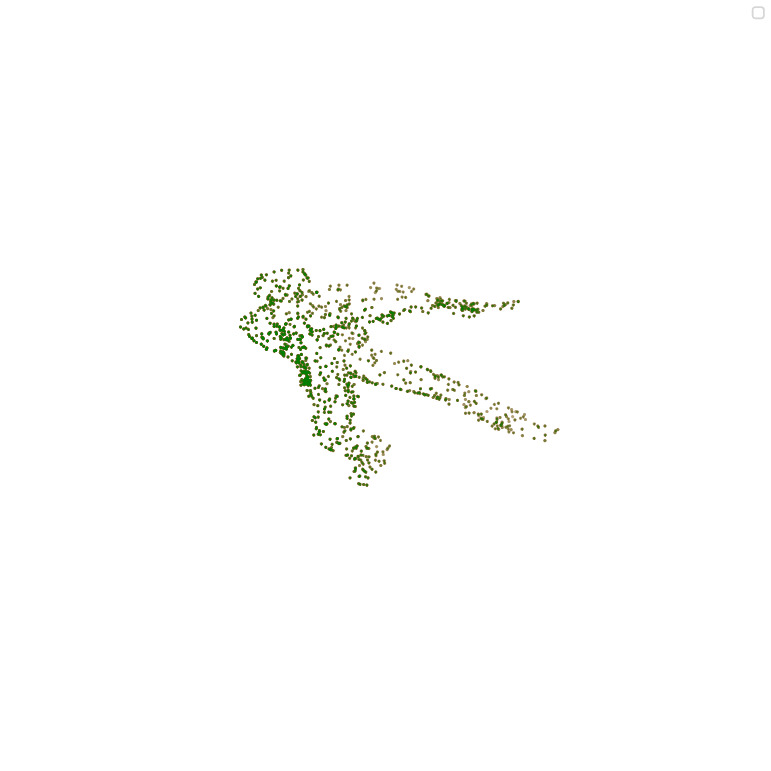} &
\includegraphics[width=0.15\textwidth]{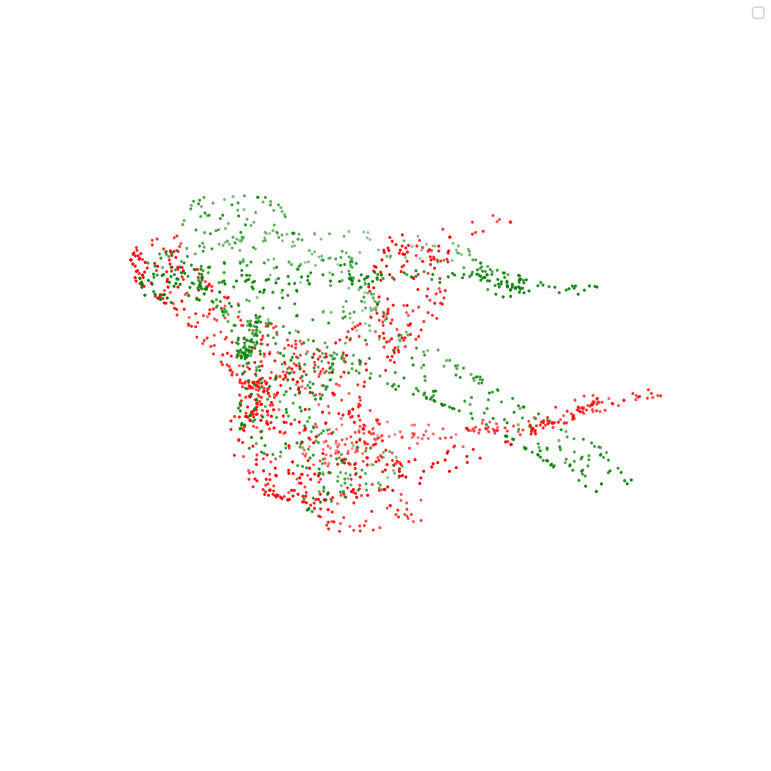} &
 \includegraphics[width=0.15\textwidth]{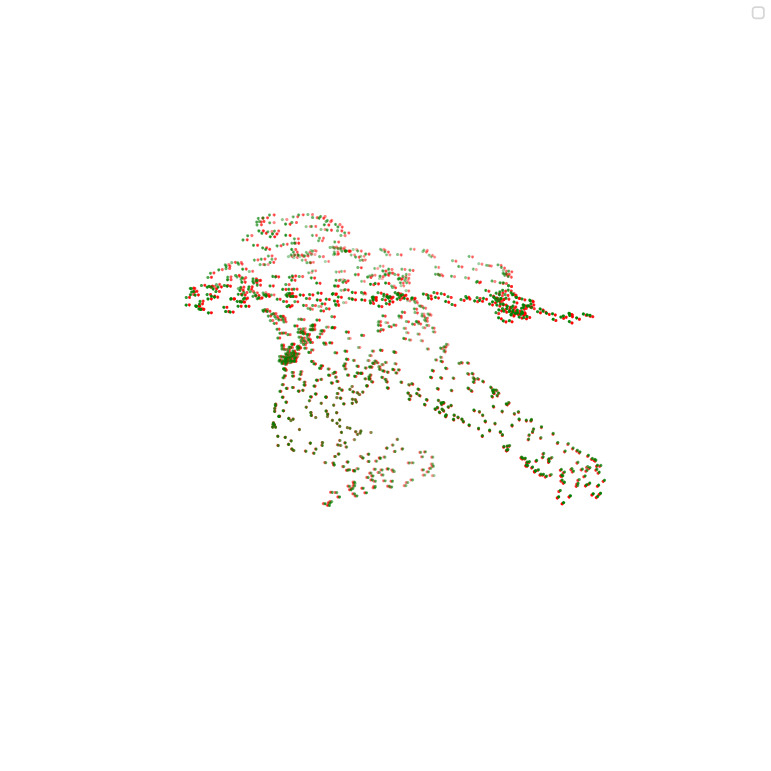} 
\\
\includegraphics[width=0.15\textwidth]{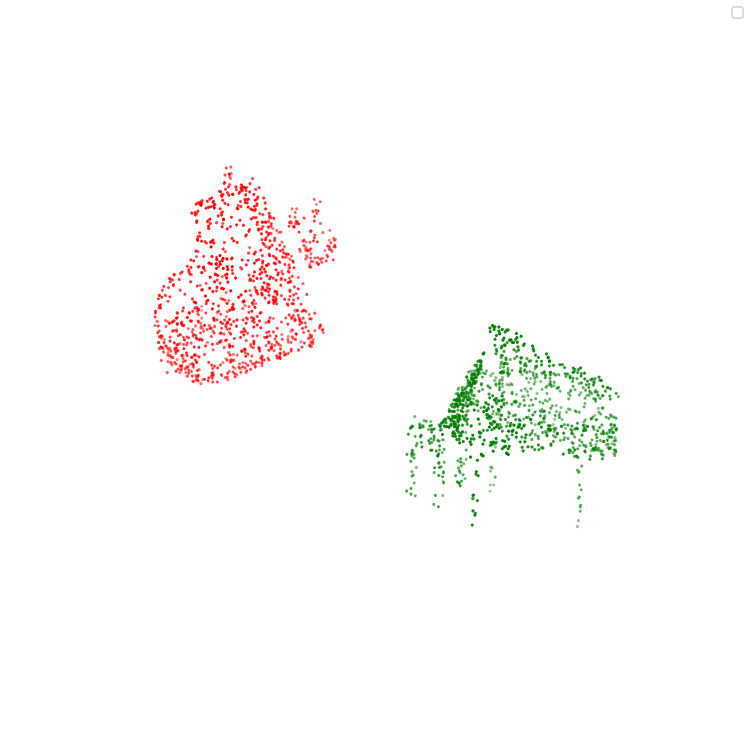} &
\includegraphics[width=0.15\textwidth]{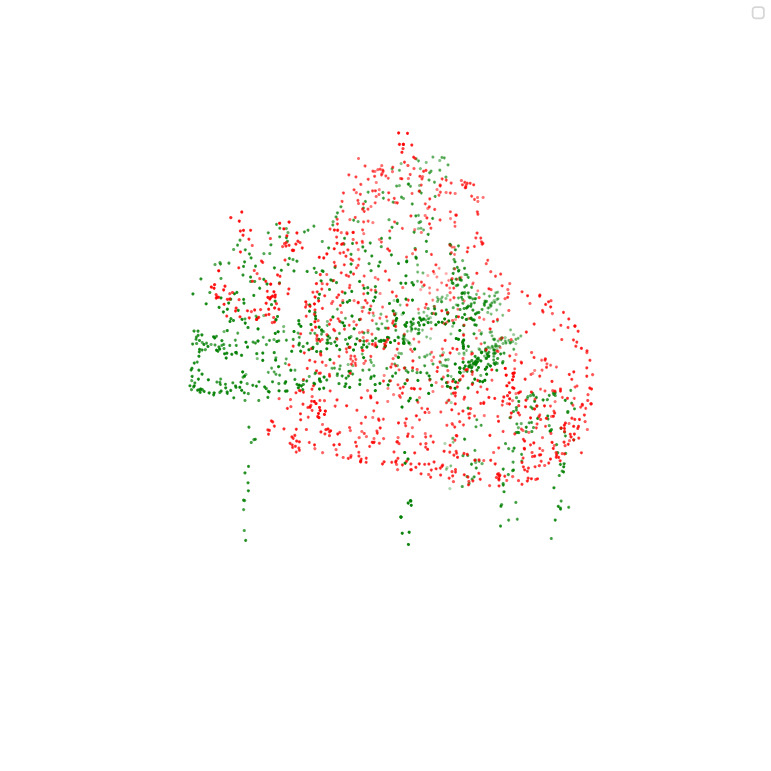} &
\includegraphics[width=0.15\textwidth]{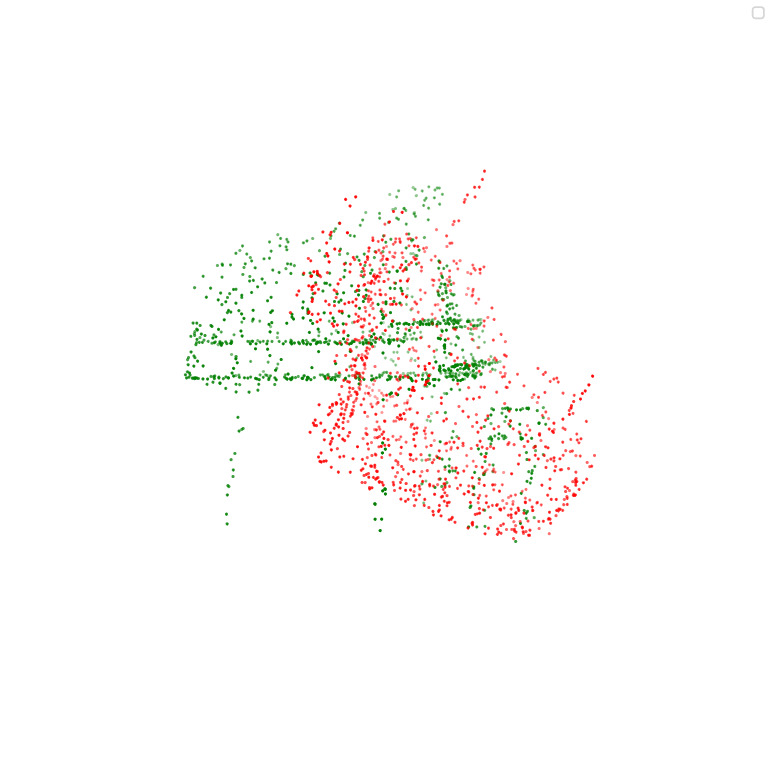} &
\includegraphics[width=0.15\textwidth]{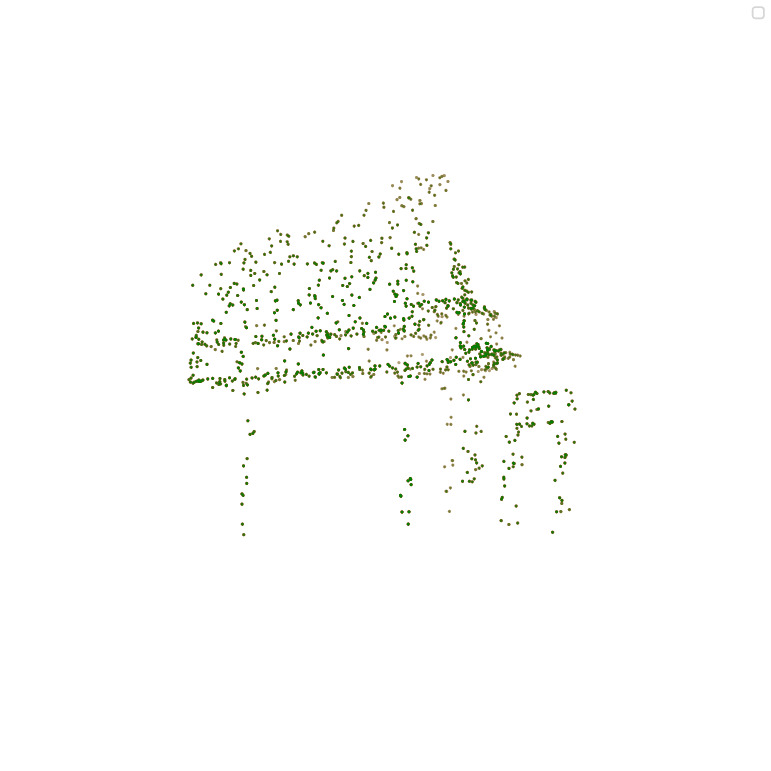} &
\includegraphics[width=0.15\textwidth]{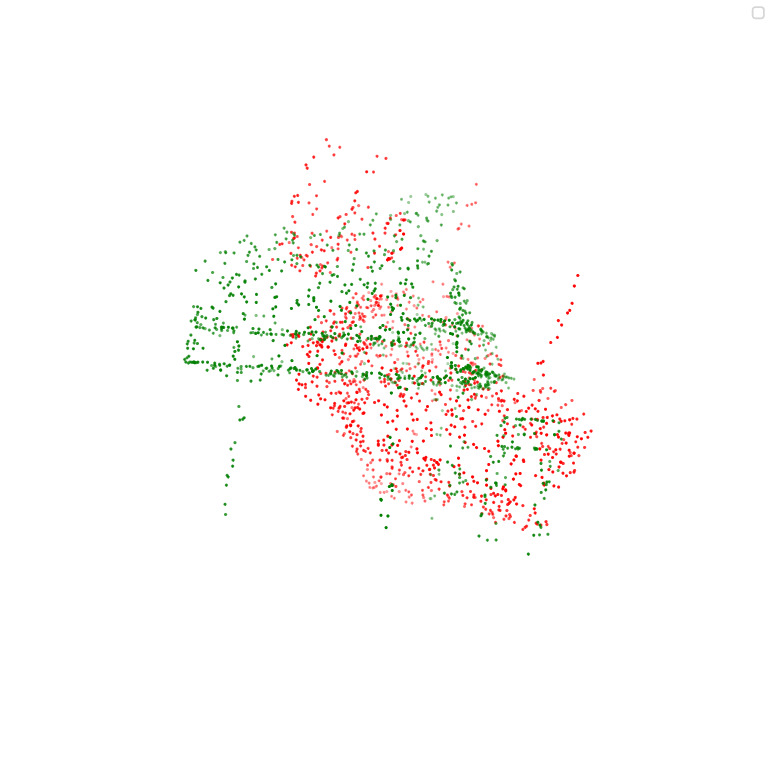} &
\includegraphics[width=0.15\textwidth]{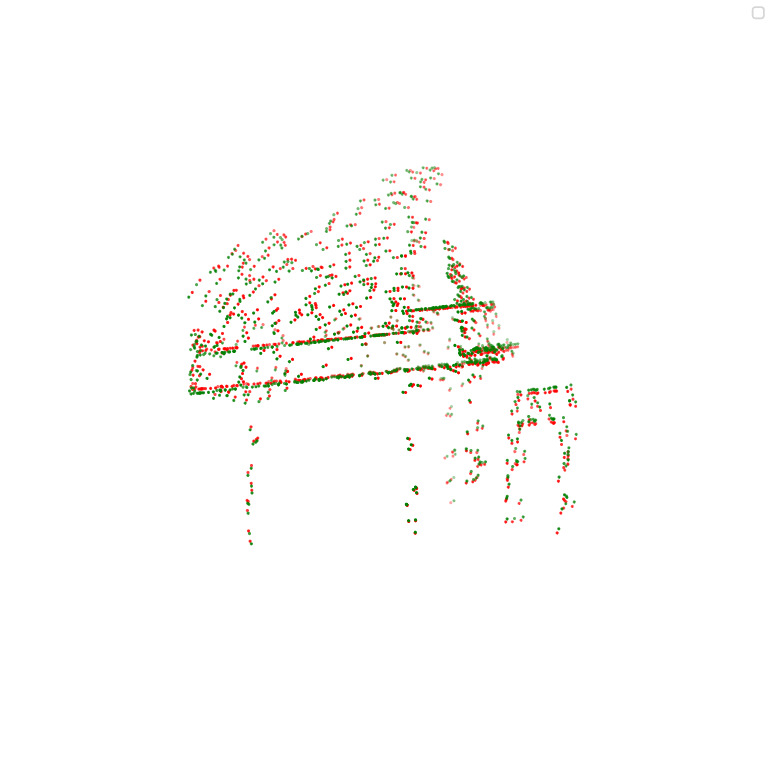} \\
\end{tabular}
\caption{Comparison of ESM-ICP alignment to other geometric-based approaches.}
\label{fig:geometric_based_approaches_comparison}
\end{figure*}

\begin{figure*}[!htb]
\centering
\renewcommand{\arraystretch}{0.0} 
\setlength{\tabcolsep}{4pt}      
\begin{tabular}{c c c c c c}
\textbf{Input} & \textbf{DCP}~\cite{wang2019deep} & \textbf{PointNetLK}~\cite{aoki2019pointnetlk} & \textbf{RpmNet}~\cite{confcvprYewL20} & \textbf{DeepGMR}~\cite{jian2005robust} & \textbf{ESM-ICP} \\
\includegraphics[width=0.15\textwidth]{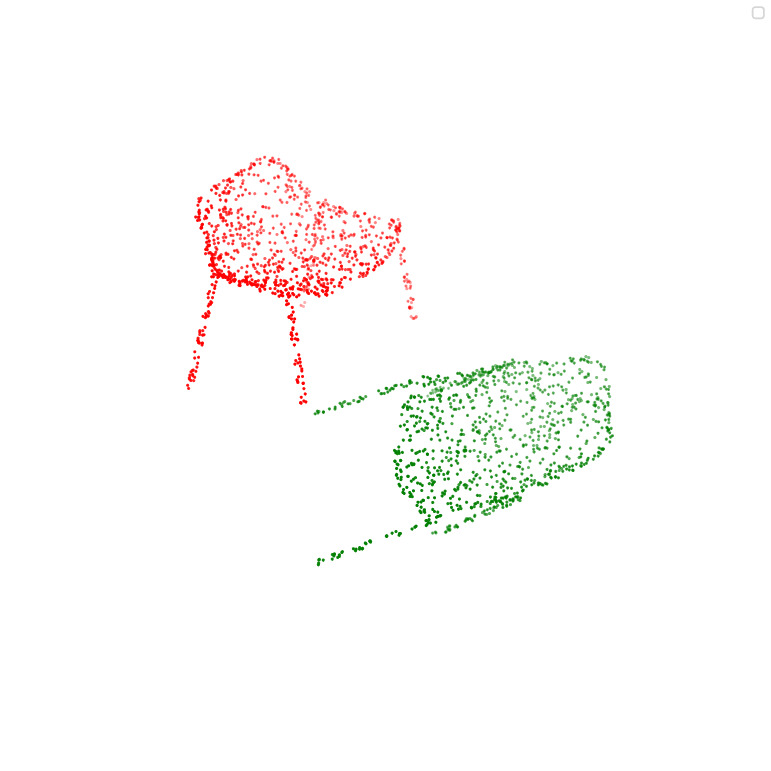} & 
\includegraphics[width=0.15\textwidth]{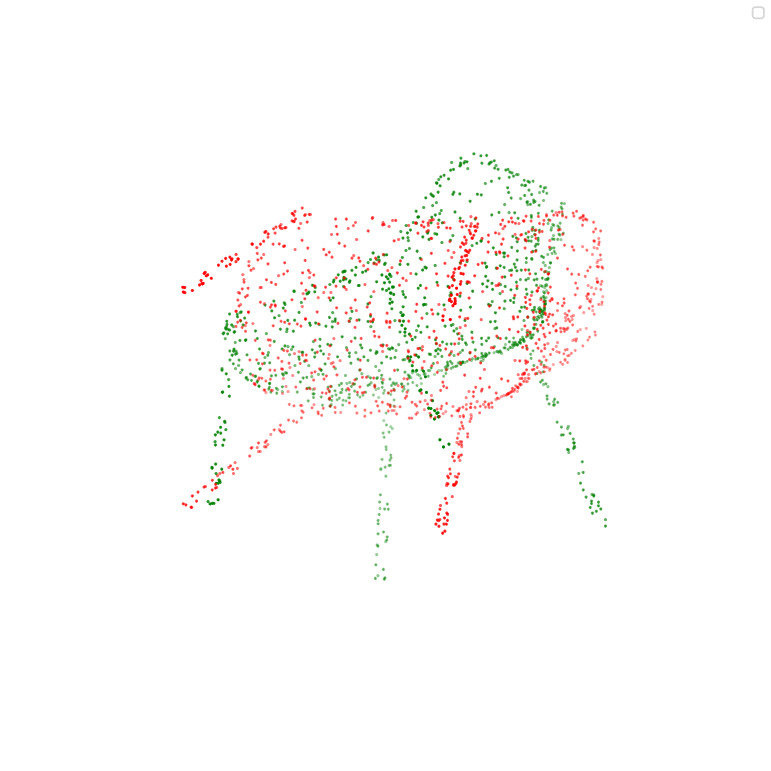} &
\includegraphics[width=0.15\textwidth]{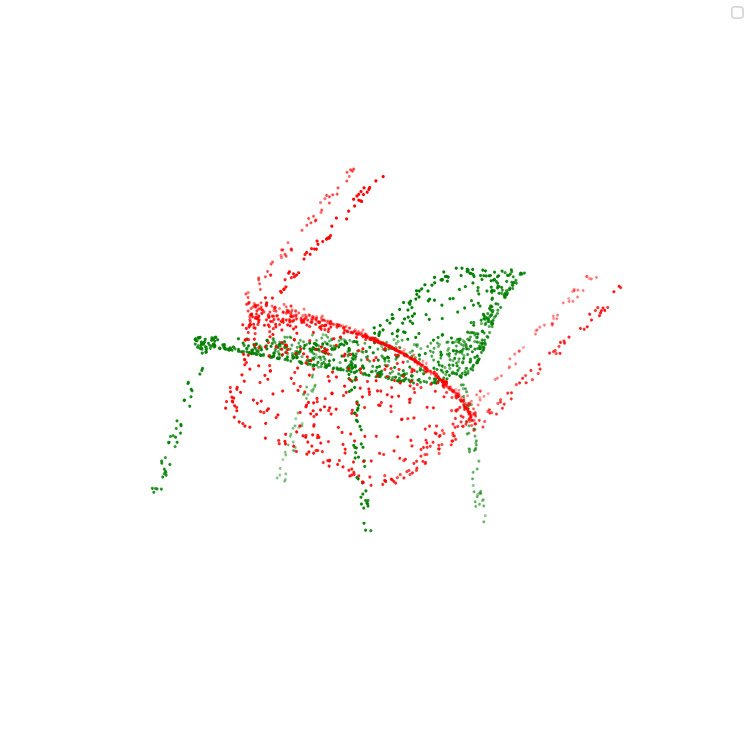} &
\includegraphics[width=0.15\textwidth]{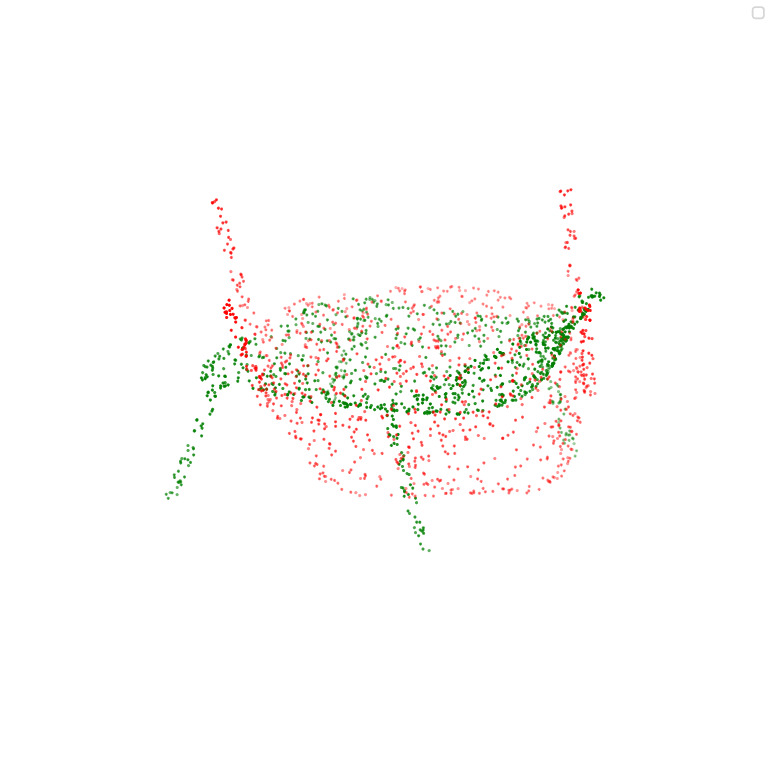} &
\includegraphics[width=0.15\textwidth]{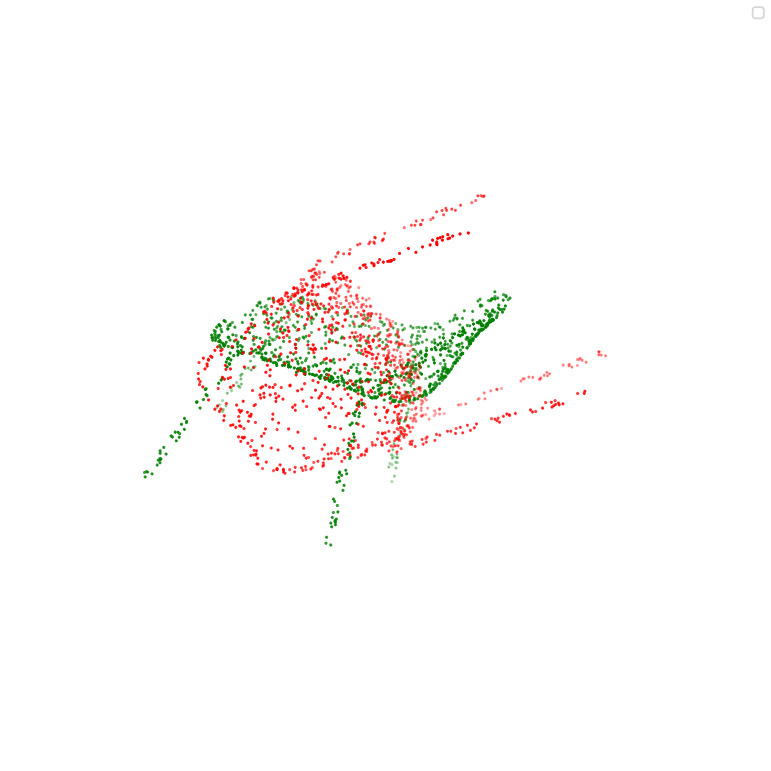} &
\includegraphics[width=0.15\textwidth]{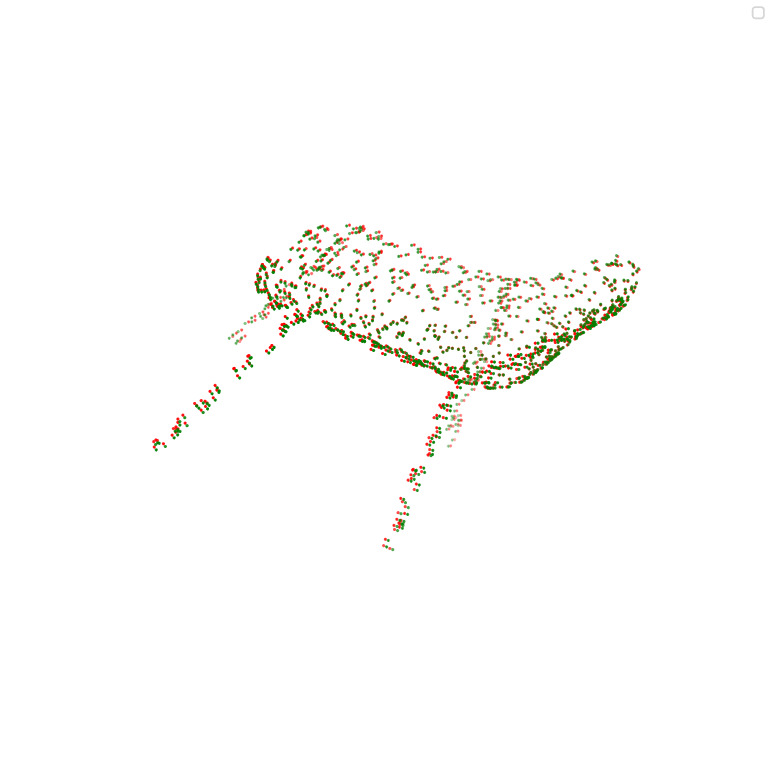} 
\\
\includegraphics[width=0.15\textwidth]{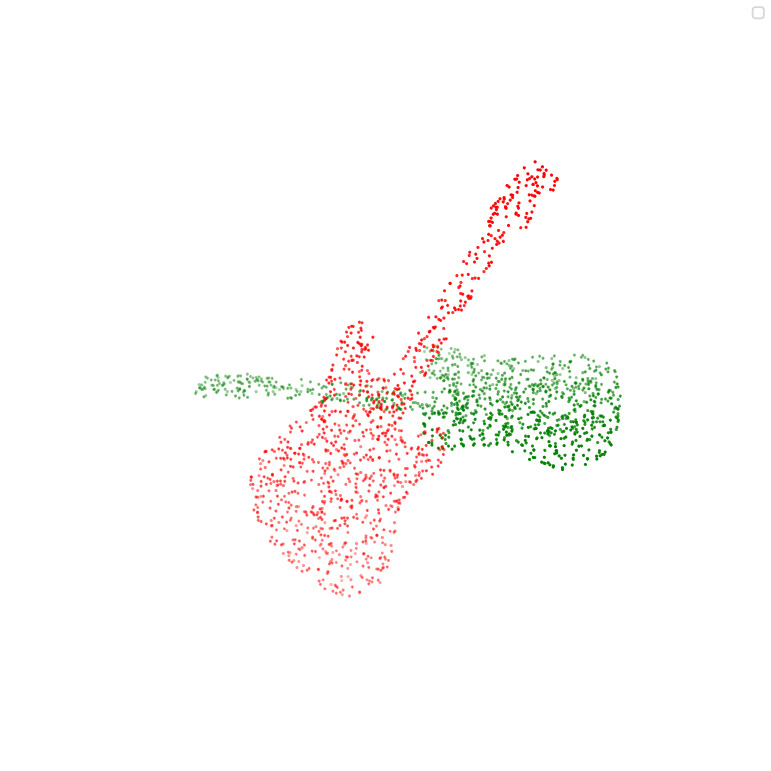} &
\includegraphics[width=0.15\textwidth]{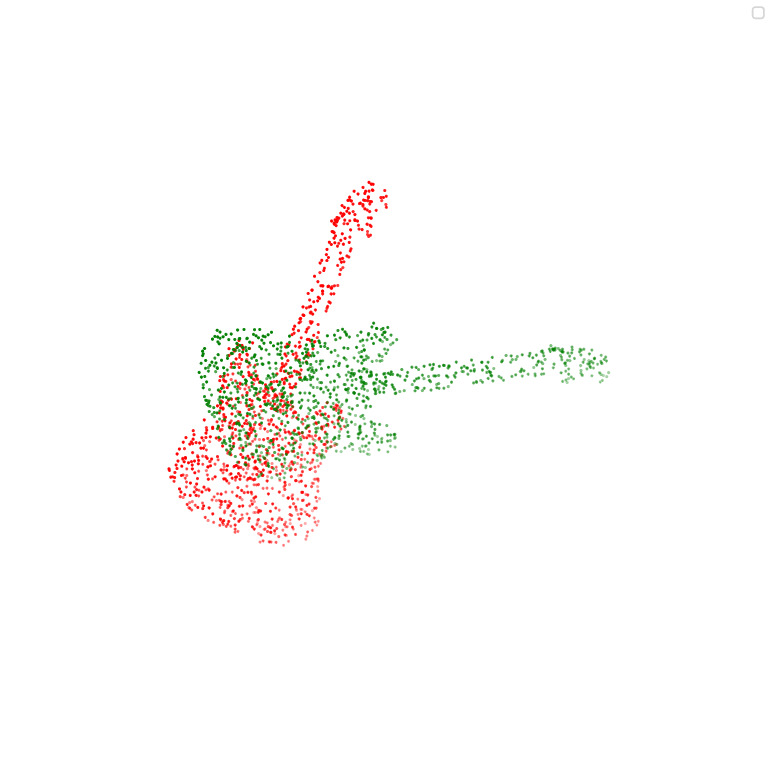} &
\includegraphics[width=0.15\textwidth]{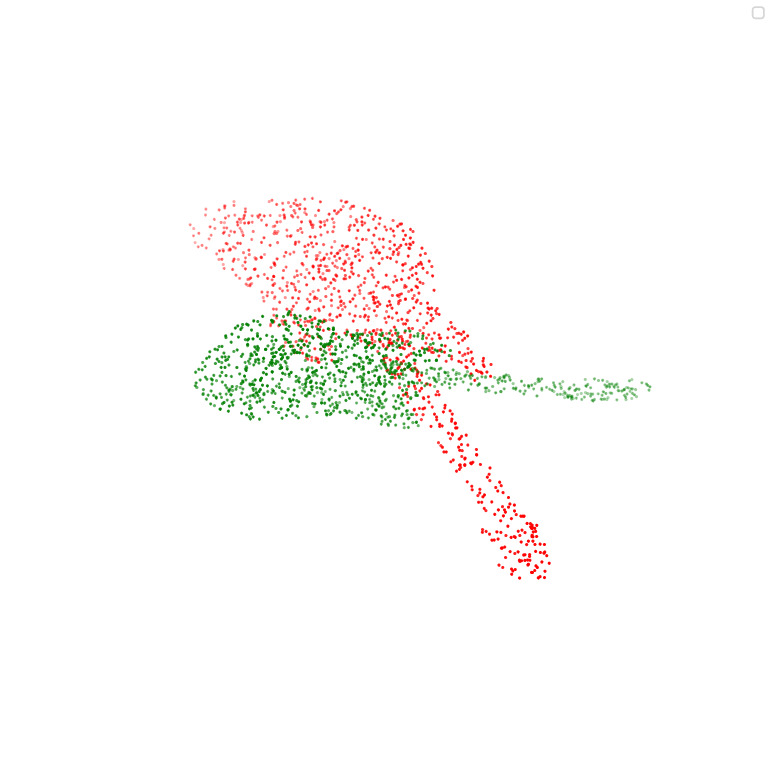} &
\includegraphics[width=0.15\textwidth]{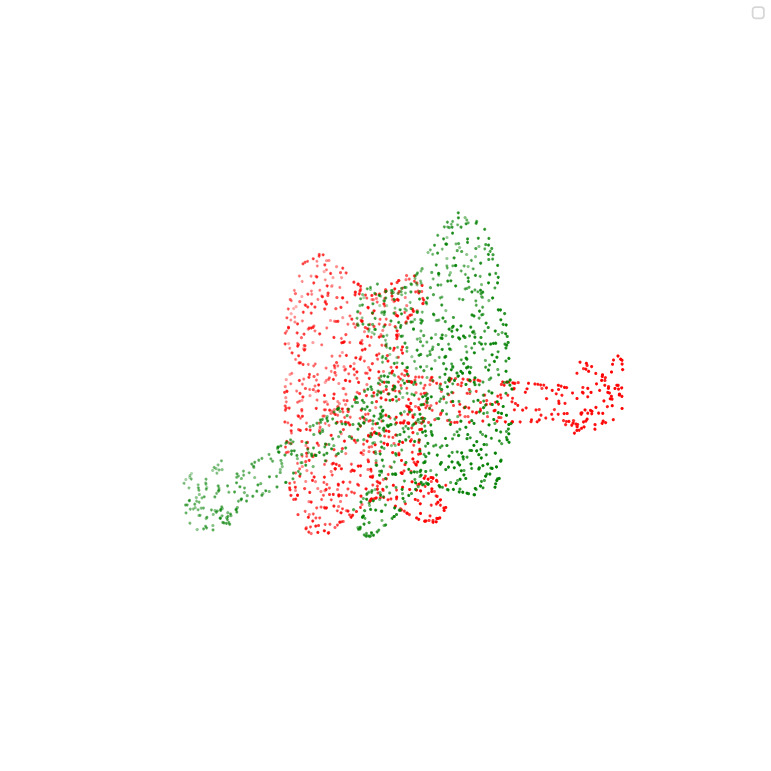} &
\includegraphics[width=0.15\textwidth]{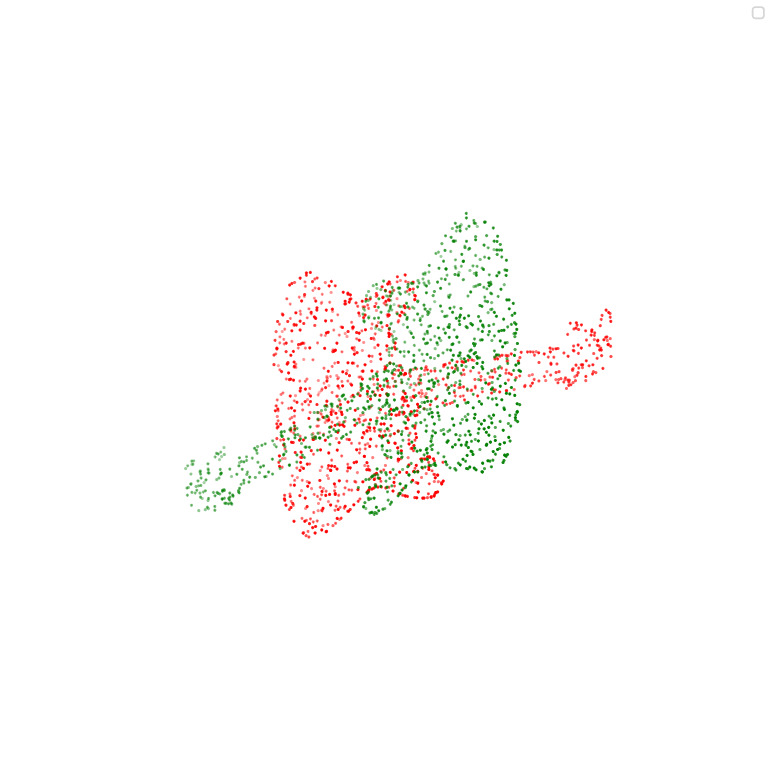} &
\includegraphics[width=0.15\textwidth]{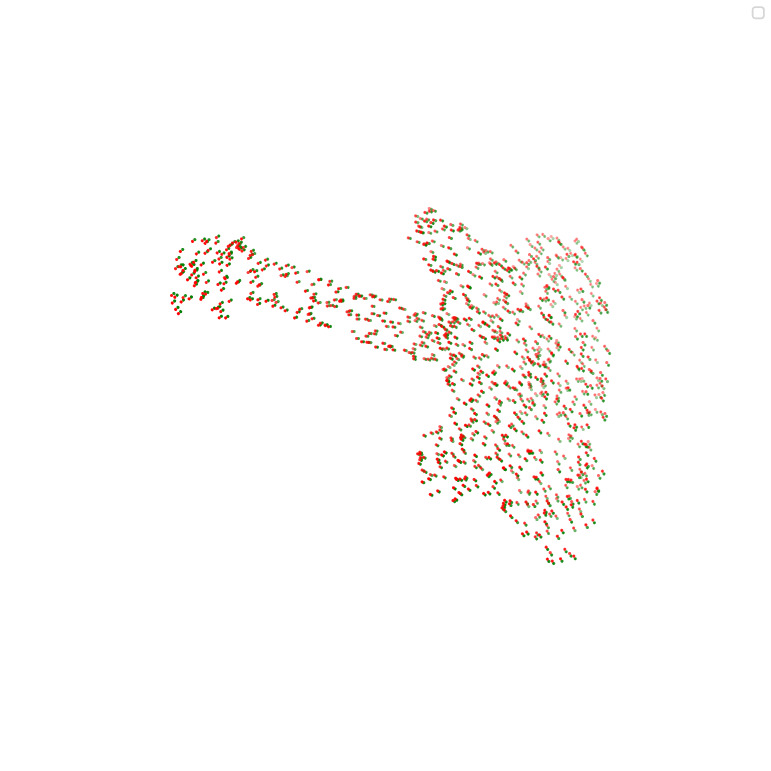} 
\\

\includegraphics[width=0.15\textwidth]{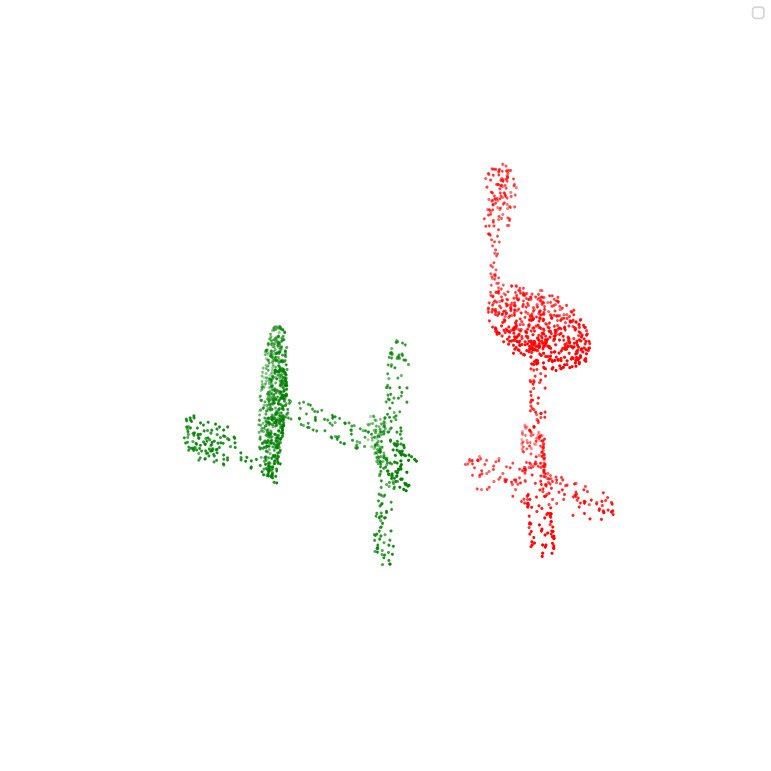} &
\includegraphics[width=0.15\textwidth]{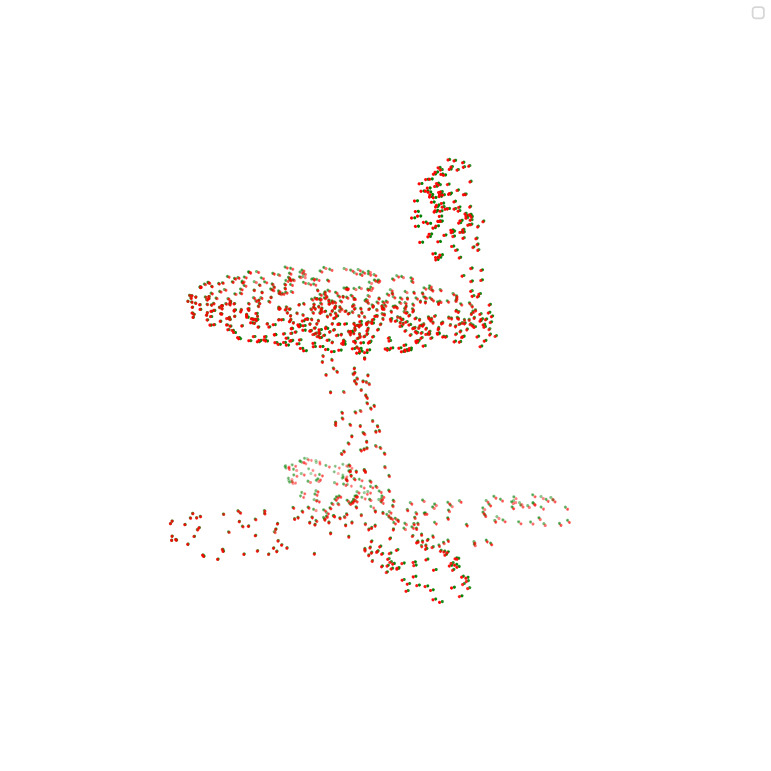} &
\includegraphics[width=0.15\textwidth]{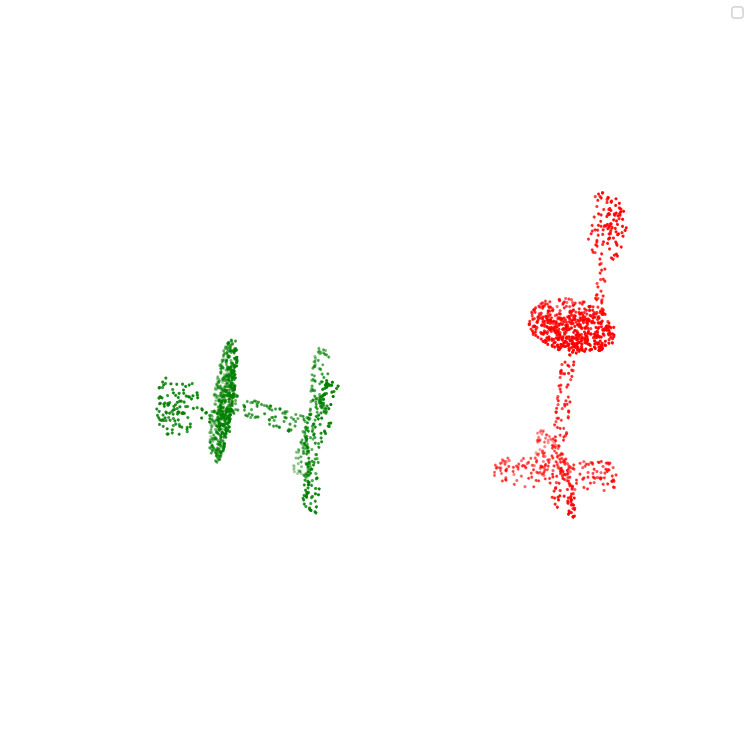} &
\includegraphics[width=0.15\textwidth]{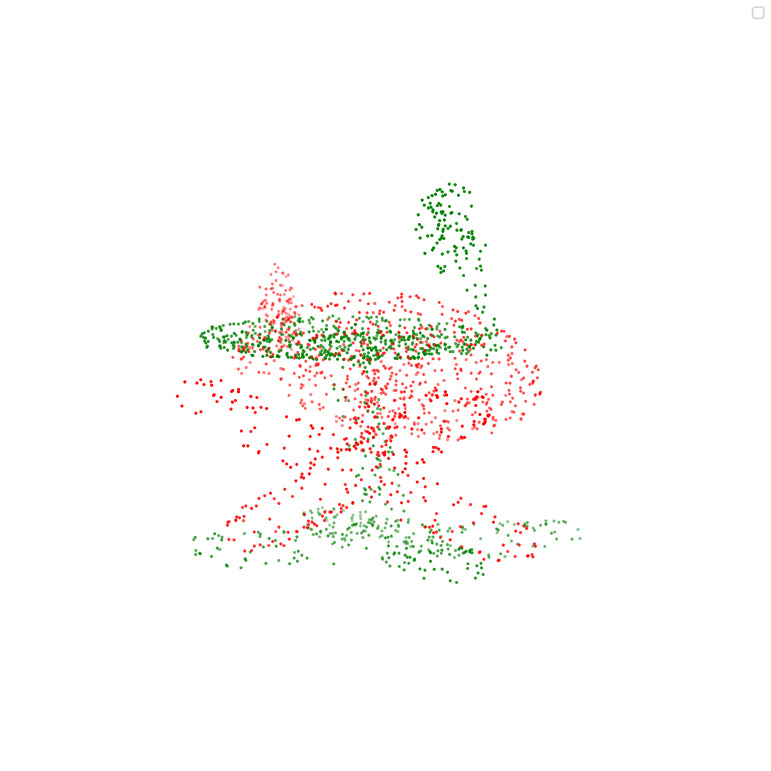} & 
\includegraphics[width=0.15\textwidth]{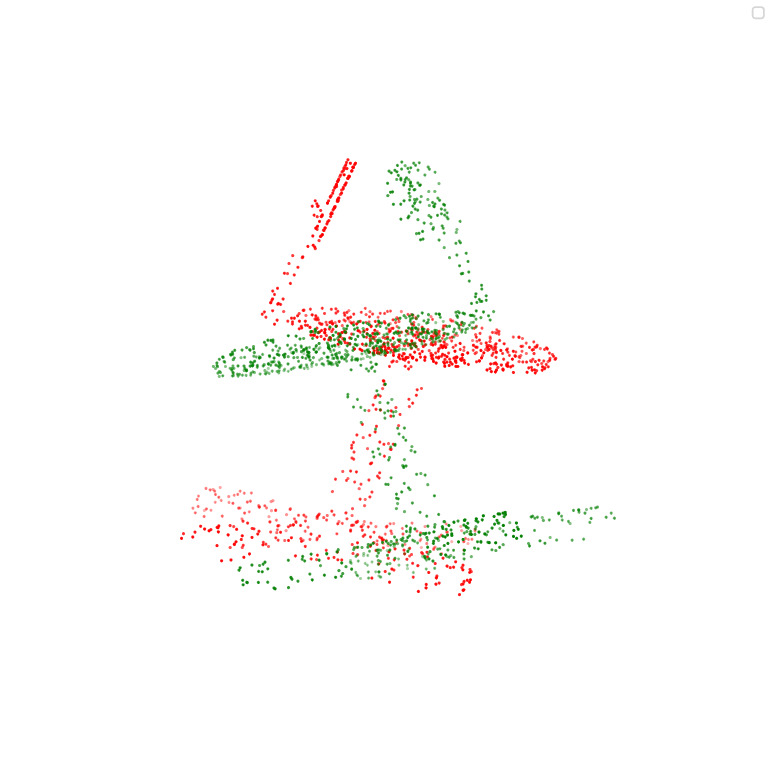}  &
\includegraphics[width=0.15\textwidth]{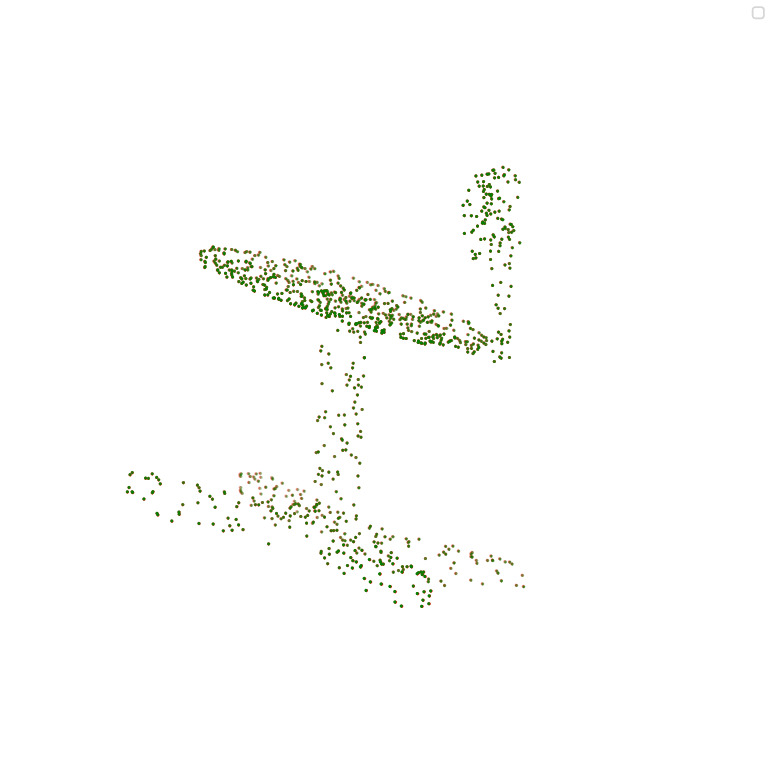} \\

\end{tabular}
\caption{Comparison to deep learning methods.}
\label{fig:ai_method_comparison}
\end{figure*}

\begin{figure*}[!htb]
\centering
\renewcommand{\arraystretch}{0.0} 
\setlength{\tabcolsep}{4pt}      
\begin{tabular}{c c c c c c } 
\textbf{} & \textbf{} & \textbf{} & \textbf{} & \textbf{} & \textbf{} \\
 
\includegraphics[width=0.15\textwidth]{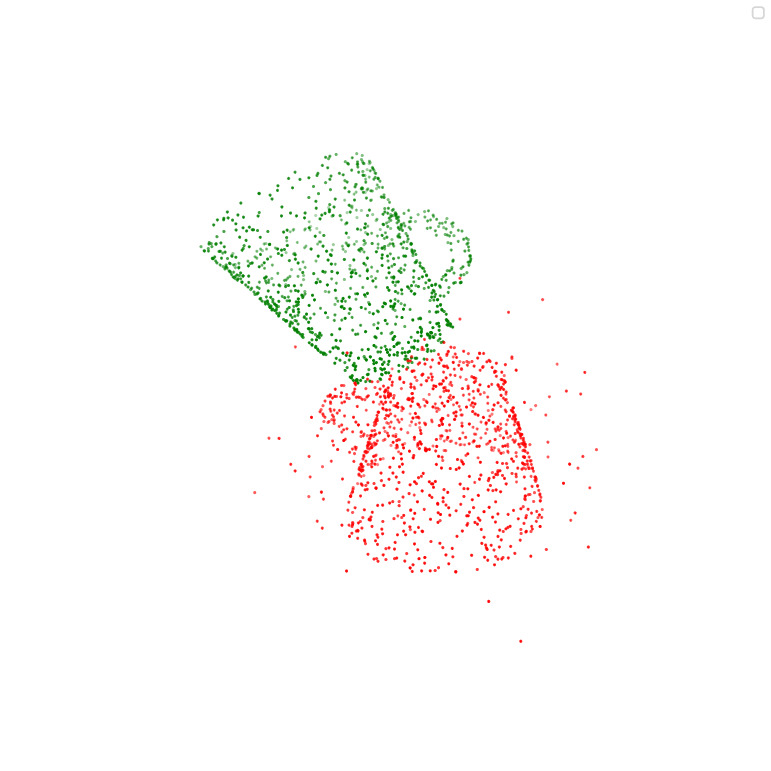} & \includegraphics[width=0.15\textwidth]{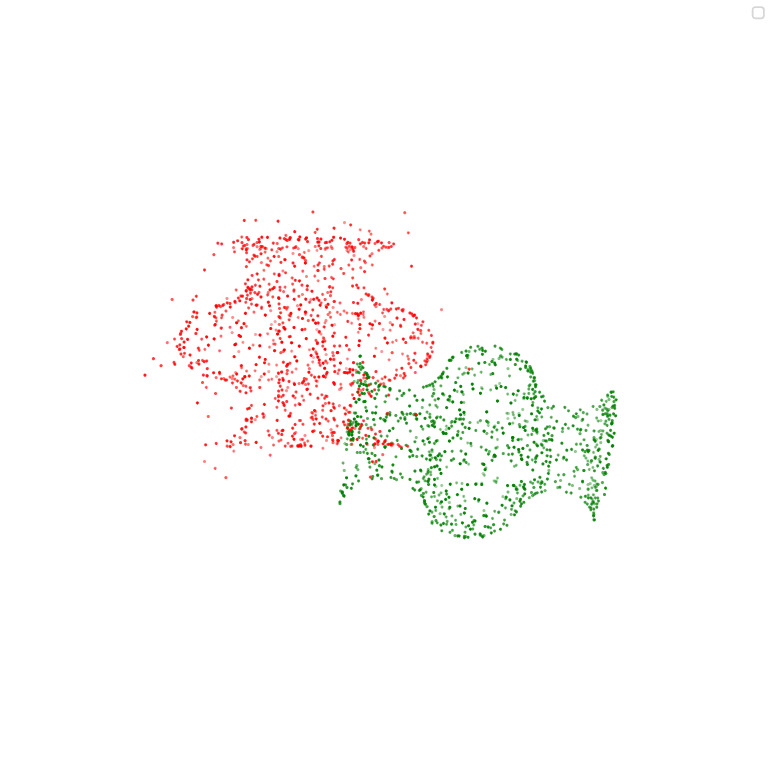} &
\includegraphics[width=0.15\textwidth]{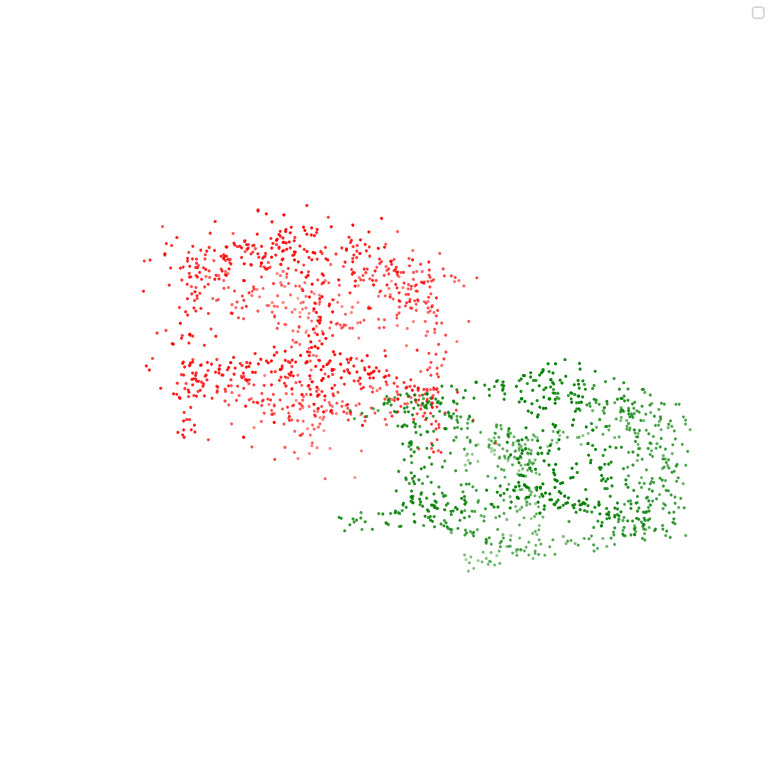} &
\includegraphics[width=0.15\textwidth]{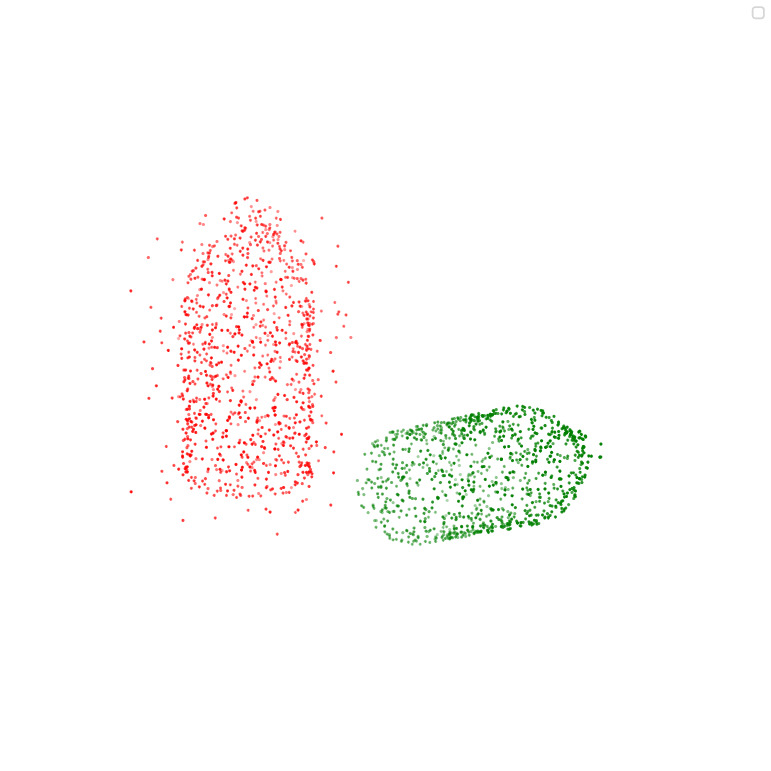} & \includegraphics[width=0.15\textwidth]{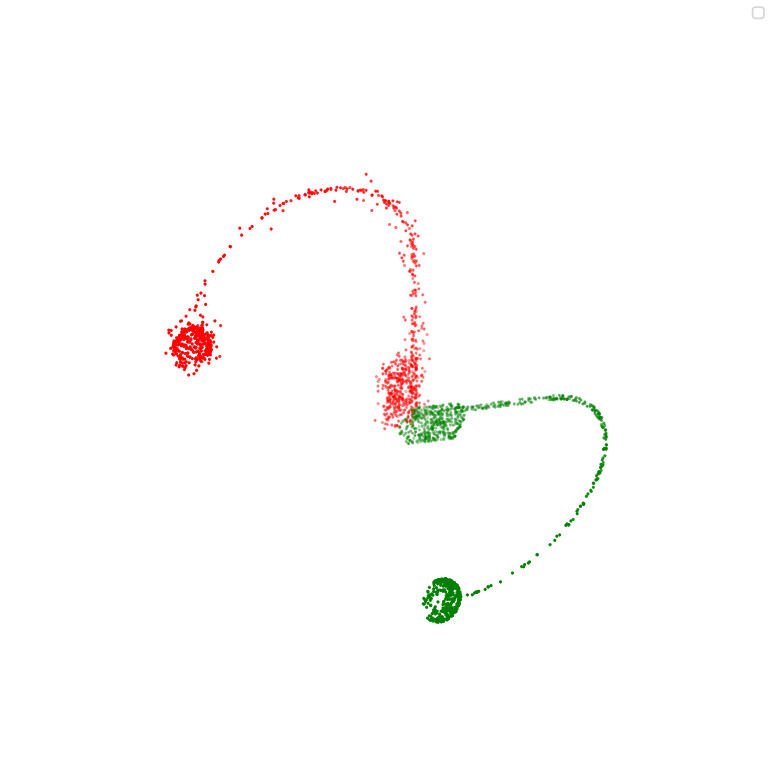} & \includegraphics[width=0.15\textwidth]{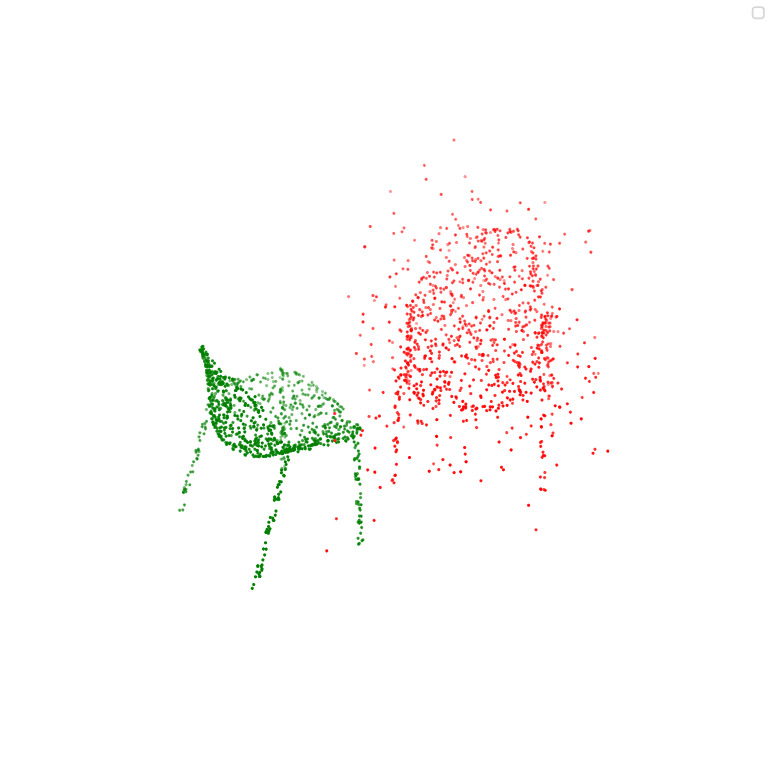} \\

\includegraphics[width=0.15\textwidth]{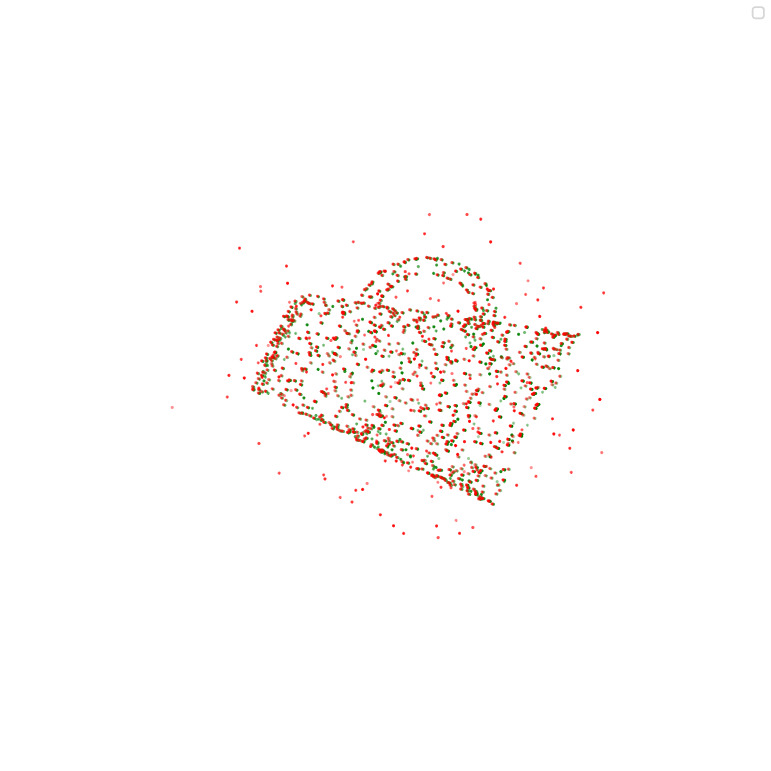} & \includegraphics[width=0.15\textwidth]{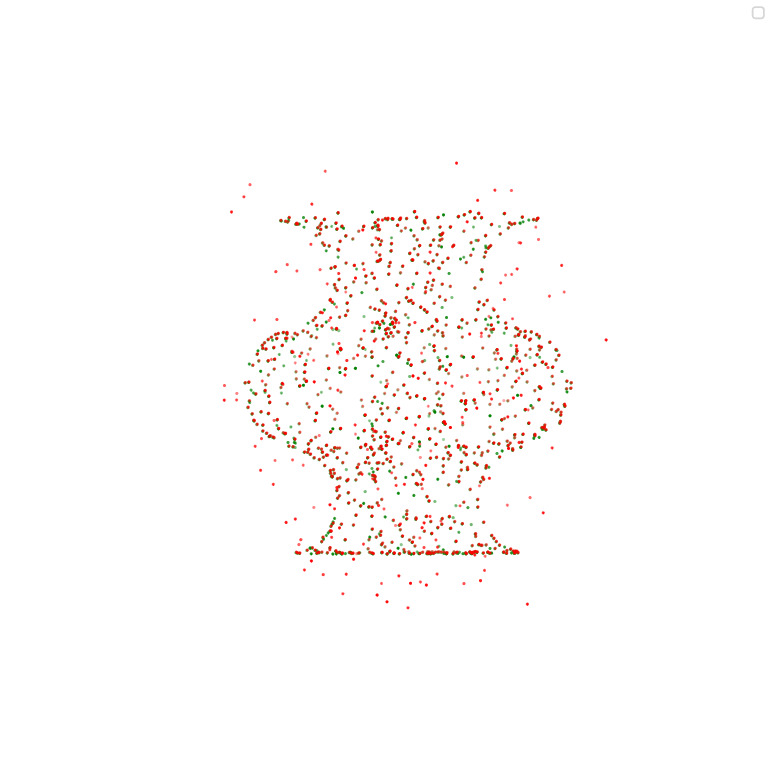} &
\includegraphics[width=0.15\textwidth]{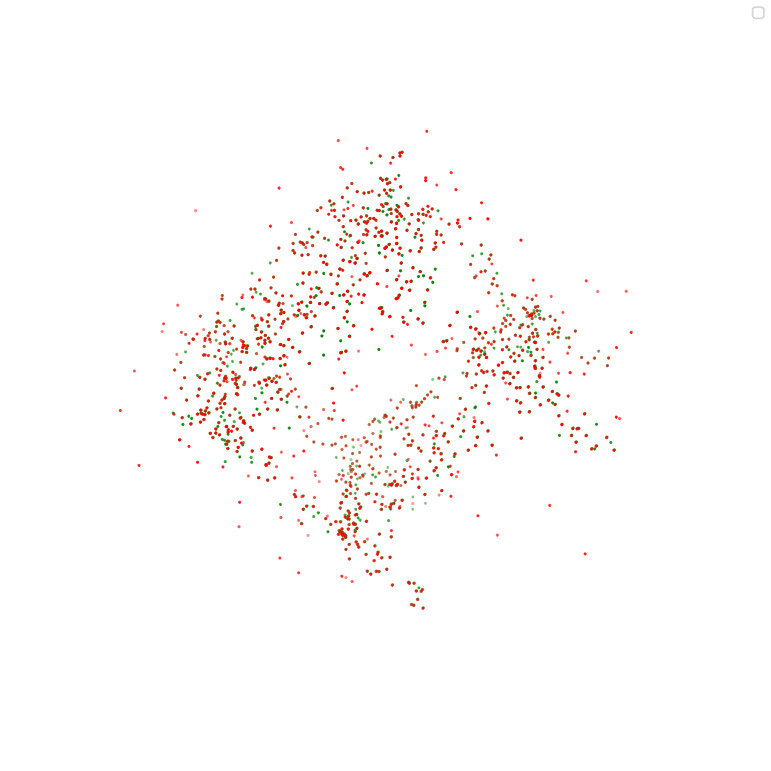} &
\includegraphics[width=0.15\textwidth]{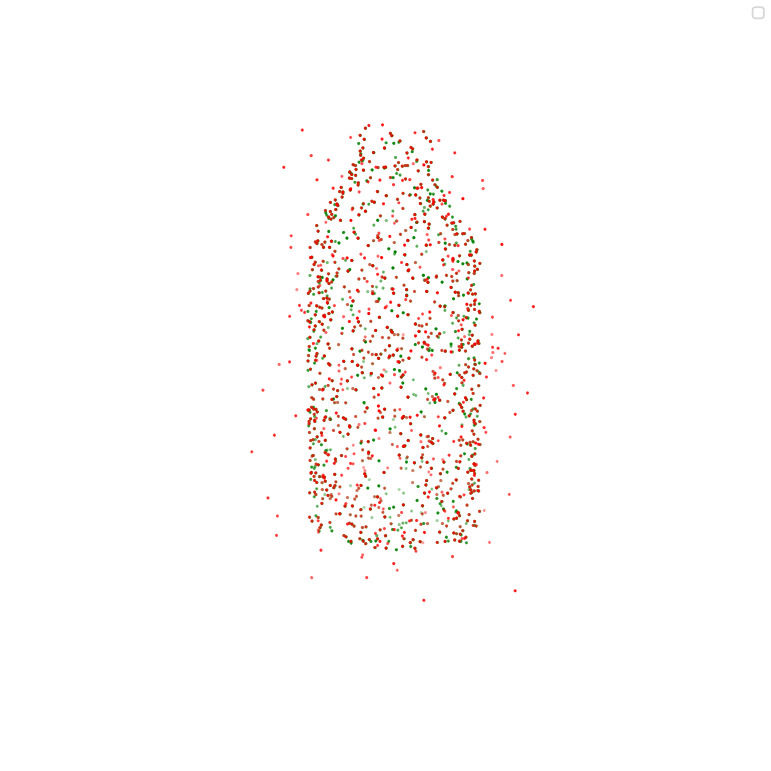} & \includegraphics[width=0.15\textwidth]{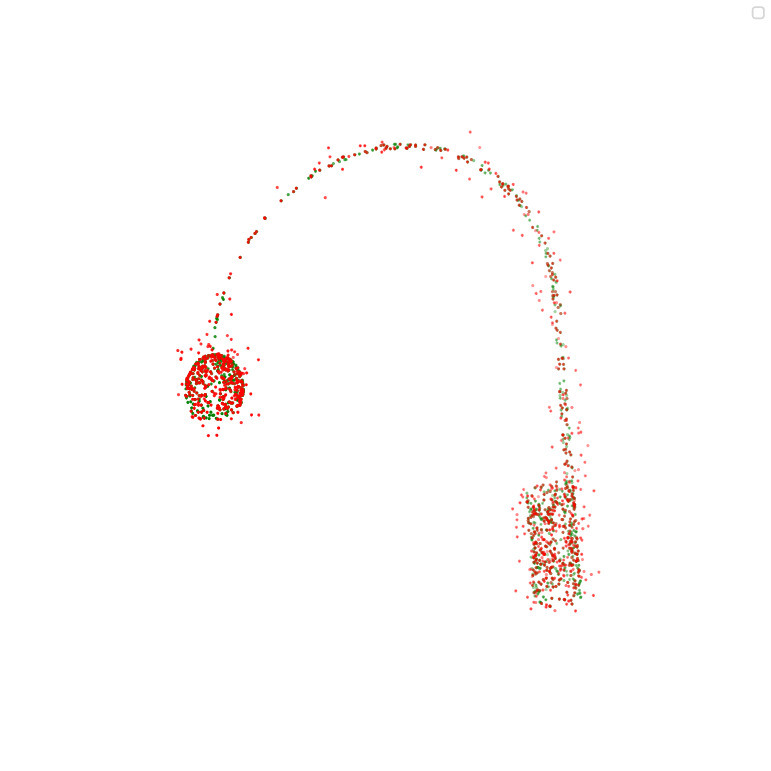} & \includegraphics[width=0.15\textwidth]{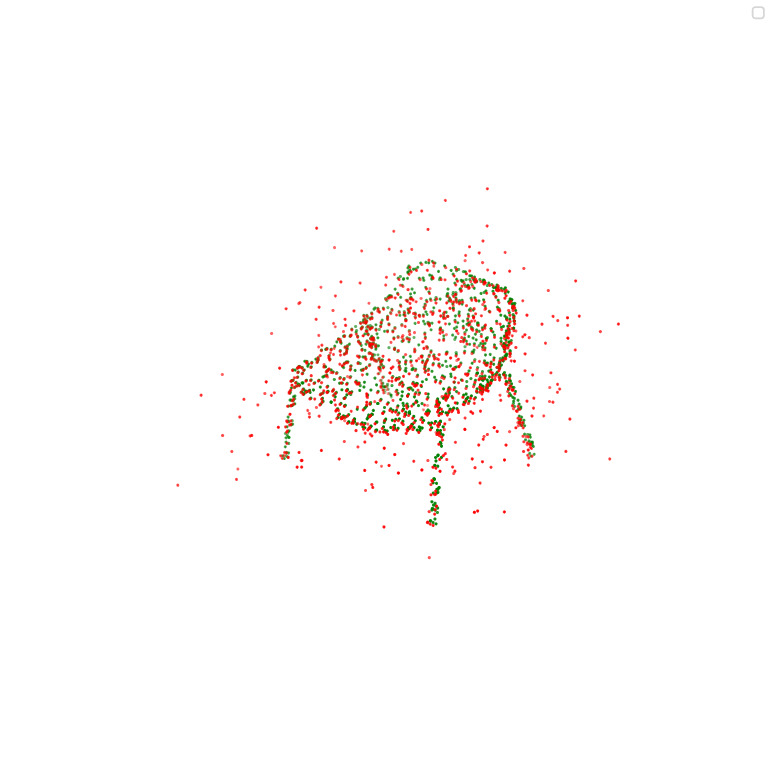} \\
\end{tabular}
\caption{Performance in the presence of outliers. The top row denotes the input, and the bottom row denotes the output from ESM-ICP.}
\label{fig:outlier_performance}
\end{figure*}

\begin{table}[h!]
\small
\centering
\caption{ModelNet40: Evaluation on Airplane dataset with random transformation matrix $(r,p,y,{\tau}_x,{\tau}_y,{\tau}_z)=(2.58175, -2.16765, 0.711617, 11.0585, 16.2285, -17.0047)$ after $40$ iterations.}
\resizebox{0.50\textwidth}{!}{  
\begin{tabular}{|l|c|c|c|c|c|c|}
\hline
\textbf{Model:Airplane} & \textbf{MSE($R$)} & \textbf{RMSE($R$)} & \textbf{MAE($R$)} & \textbf{MSE($t$)} & \textbf{RMSE($t$)} & \textbf{MAE($t$)} \\
\hline
ICP~\cite{10.1109/TPAMI.1987.4767965,121791}          & 0.888774 & 0.942748 & 0.722478 & 444942 & 667.04 & 0.722478 \\
ICP P2PL~\cite{articleptp}     & 0.887952 & 0.942312 & 0.627701 & 396629 & 629.785 & 535.349 \\
ICP NL~\cite{granger2002em,zhou2015affine}       & 0.861814 & 0.928339 & 0.725019 & 227.584 & 15.0859 & 14.7941 \\
GO-ICP~\cite{yang2015goicp} & 0.711808 & 0.843687 & 0.701879 & 227.839 & 15.0943 & 14.743 \\
RpmNET~\cite{confcvprYewL20}        & 0.033537 & 0.183130 & 0.151623 & 14822.6 & 121.748 & 107.798 \\
PointNetLK~\cite{aoki2019pointnetlk}   & 0.883079 & 0.939723 & 0.721705 & 321.678 & 17.9354 & 17.0547 \\
DCP~\cite{wang2019deep}          & 0.537342 & 0.733036 & 0.629785 & 330673 & 575.042 & 469.834 \\
\hline
\textbf{ESM-ICP (ours)} & \textbf{4.749e-16} & \textbf{2.179e-08} & \textbf{1.076e-08} & \textbf{7.549e-11} & \textbf{8.688e-06} & \textbf{7.947e-06} \\
\hline
\end{tabular}
}
\label{tab:modelnet40_airplane_results}
\end{table}

\begin{table}[h!]
\small
\centering
\caption{ModelNet40: Evaluation on Person dataset with random transformation matrix $(r,p,y,{\tau}_x,{\tau}_y,{\tau}_z)=(-2.56814, 1.15753, 1.92252, 0.0635863, 0.0883894, 0.0158466)$ after $40$ iterations.}
\resizebox{0.50\textwidth}{!}{  
\begin{tabular}{|l|c|c|c|c|c|c|}
\hline
\textbf{Model:} & \textbf{MSE($R$)} & \textbf{RMSE($R$)} & \textbf{MAE($R$)} & \textbf{MSE($t$)} & \textbf{RMSE($t$)} & \textbf{MAE($t$)} \\
\hline
ICP~\cite{10.1109/TPAMI.1987.4767965,121791}          & 0.888147 & 0.942415 & 0.811008 & 1823.66 & 42.7044 & 25.7527 \\
ICP P2PL~\cite{articleptp}     & 0.708472 & 0.841708 & 0.721835 & 1431.99 & 37.8417 & 29.8622 \\
ICP NL~\cite{granger2002em,zhou2015affine}        & 0.812057 & 0.901142 & 0.808517 & 0.0107494 & 0.103679 & 0.808517\\
GO-ICP~\cite{yang2015goicp}       & 8.95827e-14 & 2.99304e-07 & 2.16895e-07 & 5.11454e-10 & 2.26154e-05 & 1.4035e-05\\
RpmNET~\cite{confcvprYewL20}       & 0.866954 & 0.931104 & 0.783011 & 969.146 & 31.1311 & 28.6128 \\
PointNetLK~\cite{aoki2019pointnetlk}    & 0.885005 & 0.940747 & 0.805605 & 65.5633 & 8.09712 & 7.25109 \\
DCP~\cite{wang2019deep}             & 0.784646 & 0.885803 & 0.666492 & 3292.4 & 57.3795 & 51.1154 \\
\hline
\textbf{ESM-ICP (ours)} & \textbf{2.46716e-16} & \textbf{1.57072e-08} & \textbf{9.93411e-09} & \textbf{1.19354e-12} & \textbf{1.09249e-06} & \textbf{9.10833e-07} \\
\hline
\end{tabular}
}
\label{tab:modelnet40_person_results}
\end{table}

\begin{table}[h!]
\small
\centering
\caption{ModelNet40: Evaluation on Vase dataset (with $30\%$ of the data affected with outliers) with random transformation matrix $(r,p,y,{\tau}_x,{\tau}_y,{\tau}_z)=(-2.51358, 0.916596, 1.10051, 105.383, -53.3576, -16.489)$ after $40$ iterations.}
\resizebox{0.50\textwidth}{!}{  
\begin{tabular}{|l|c|c|c|c|c|c|}
\hline
\textbf{Model:} & \textbf{MSE($R$)} & \textbf{RMSE($R$)} & \textbf{MAE($R$)} & \textbf{MSE($t$)} & \textbf{RMSE($t$)} & \textbf{MAE($t$)} \\
\hline
ICP~\cite{10.1109/TPAMI.1987.4767965,121791}          & 0.864082 & 0.92956 & 0.794351 & 23136.8 & 152.108 & 125.499 \\
ICP P2PL~\cite{articleptp}      & 0.438943 & 0.662528 & 0.50979 & 24721.3 & 157.23 & 135.133 \\
ICP NL~\cite{granger2002em,zhou2015affine}        & 0.888659 & 0.942687 & 0.813719 & 4741.5 & 68.8585 & 58.4099\\
GO-ICP~\cite{yang2015goicp}       & 0.0755686 & 0.274897 & 0.240074 & 4173.73 & 64.6044 & 56.6411\\
RpmNET~\cite{confcvprYewL20}       & 0.887304 & 0.941968 & 0.8128 & 16041.9 & 126.657 & 112.902\\
PointNetLK~\cite{aoki2019pointnetlk}    & 0.882742 & 0.939544 & 0.738511 & 4548.29 & 67.441 & 60.0684  \\
DCP~\cite{wang2019deep}             & 0.888659 & 0.942687 & 0.813719 & 84140.6 & 290.07 & 245.192 \\
\hline
\textbf{ESM-ICP (ours)} & \textbf{2.26497e-07} & \textbf{0.000475917} & \textbf{0.000384784} & \textbf{0.13024} & \textbf{0.360888} & \textbf{0.356485} \\
\hline
\end{tabular}
}
\label{tab:modelnet40_vase_outliers_results}
\end{table}

\begin{table}[h!]
\small
\centering
\caption{ModelNet40: Evaluation on Bottle dataset (with $60\%$ of the data affected with outliers) with random transformation matrix $(r,p,y,{\tau}_x,{\tau}_y,{\tau}_z)=(2.82862, -0.905772, 1.88224, -62.3639, -31.1694, -56.1162)$ after $40$ iterations.}
\resizebox{0.50\textwidth}{!}{  
\begin{tabular}{|l|c|c|c|c|c|c|}
\hline
\textbf{Model:} & \textbf{MSE($R$)} & \textbf{RMSE($R$)} & \textbf{MAE($R$)} & \textbf{MSE($t$)} & \textbf{RMSE($t$)} & \textbf{MAE($t$)} \\
\hline
ICP~\cite{10.1109/TPAMI.1987.4767965,121791}          & 0.888844 & 0.942785 & 0.839232 & 15541 & 124.663 & 110.264 \\
ICP P2PL~\cite{articleptp}      & 0.888172 & 0.942429 & 0.842185 & 12393.2 & 111.325 & 80.3513\\
ICP NL~\cite{granger2002em,zhou2015affine}        & 0.826324 & 0.909024 & 0.909024 & 2648.4 & 51.4626 & 49.6877 \\
GO-ICP~\cite{yang2015goicp}       & 0.130485 & 0.361228 & 0.309347 & 2777.48 & 52.7018 & 51.1383 \\
RpmNET~\cite{confcvprYewL20}       & 0.679944 & 0.824587 & 0.681382 &  5850.54 & 76.4888  & 60.0398 \\
PointNetLK~\cite{aoki2019pointnetlk}    & 0.885858 & 0.9412 & 0.735128 & 2695.98 & 51.9228 & 50.5983\\
DCP~\cite{wang2019deep}             & 0.825616 & 0.908634 & 0.792235 & 11491 & 107.196 & 81.2461 \\
\hline
\textbf{ESM-ICP (ours)} & \textbf{1.90873e-05} & \textbf{0.00436891} & \textbf{0.00354401} & \textbf{0.220783} & \textbf{0.469875} & \textbf{0.424405} \\
\hline
\end{tabular}
}
\label{tab:modelnet40_bottle_outliers_results}
\end{table}

\subsection{Effect of $\sigma$.}
As highlighted in the previous section, our proposed method consistently demonstrates strong performance across various ModelNet40 datasets when the parameter $\sigma$ is set to values between $0.05$ and $1.0$. This observation suggests that the effectiveness of $\sigma$ is inherently linked to the density of the point cloud data. For point clouds with higher spatial resolution, a smaller $\sigma$ tends to be sufficient, whereas lower-density data may benefit from a larger $\sigma$ to preserve alignment stability. Identifying an adaptive or data-driven strategy to select the optimal $\sigma$ for different sensing modalities and application domains remains a promising direction for future research.

\subsection{Comparison with KISS-Matcher}
During the final stages of this work, we became aware of another algorithm for point set registration known as KISS-Matcher~\cite{lim2025icra-KISSMatcher}. KISS-Matcher~\cite{lim2025icra-KISSMatcher} introduces a novel feature detector, Faster-FPFH, and incorporates a $k$-core based graph-theoretic pruning strategy to efficiently reject outlier correspondences and reduce computational complexity. Recognizing its potential, we evaluated KISS-Matcher~\cite{lim2025icra-KISSMatcher} under random transformations and observed that it exhibited strong robustness to large rotations, comparable to our own results. To further strengthen our validation, we conducted a comparative analysis between our method and KISS-Matcher~\cite{lim2025icra-KISSMatcher} across the ModelNet40 datasets (airplane, person, and stool).

We performed regression testing by loading point clouds as \textit{Source} and applying approximately 1000 random transformations — with rotations sampled from $[-\pi, \pi]$ and translations from $[-1000, 1000]$ — to generate the corresponding \textit{Target}, consistent with our prior experiments. For our method, we capped the iteration count at 100. Across these various transformations, both KISS-Matcher~\cite{lim2025icra-KISSMatcher} and our proposed ESM-ICP method demonstrated successful convergence. However, when comparing the mean error metrics for both rotation and translation components, ESM-ICP consistently outperformed KISS-Matcher~\cite{lim2025icra-KISSMatcher}. Table~\ref{tab:modelnet40_kiss_mathcer_vs_ESM-ICP} presents the quantitative comparison of the two methods. Furthermore, ESM-ICP achieved significantly superior performance in scenarios with outliers, as detailed in Table~\ref{tab:stool_kiss_mathcer_vs_ESM-ICP_30_per}.

\begin{table}[h!]
\small
\centering
\caption{Comparative analysis of ESM-ICP and KISSMatcher~\cite{lim2025icra-KISSMatcher} across 1000 different transformations ranging from $[-\pi,\pi]$ rotation wise and $[-1000,1000]$ translation wise. The metrics are computed as mean of 1000 different transformations.}
\resizebox{0.50\textwidth}{!}{  
\begin{tabular}{|l|c|c|c|c|c|c|}
\hline
\textbf{Model:} & \textbf{MSE($R$)} & \textbf{RMSE($R$)} & \textbf{MAE($R$)} & \textbf{MSE($t$)} & \textbf{RMSE($t$)} & \textbf{MAE($t$)} \\
\hline
Airplane (KISS-Matcher)~\cite{lim2025icra-KISSMatcher} & 3.77e-09 & 0.000007 & 0.000005 & 0.001126 & 0.003973 & 0.003349 \\
Person (KISS-Matcher)~\cite{lim2025icra-KISSMatcher} & 0.000231 & 0.000507 & 0.000428 & 0.930943 & 0.032311 & 0.023097 \\
Stool (KISS-Matcher)~\cite{lim2025icra-KISSMatcher} & 0.000004 & 0.001152 & 0.000936 & 0.000003  & 0.000975 & 0.000847 \\
\hline
Airplane (\textbf{ESM-ICP (ours)}) & 2.28e-09 & 0.000002 & 0.000001 & 0.000690 & 0.000890 & 0.000821 \\
Person (\textbf{ESM-ICP (ours)}) & 1.832808e-11 & 2.424087e-06 & 1.970518e-06 &  3.173788e-08 & 1.074056e-04  & 9.120257e-05 \\
Stool (\textbf{ESM-ICP (ours)})    & 1.407338e-07 & 2.416280e-04 & 1.955315e-04 & 8.803378e-08 & 1.924926e-04 & 1.623945e-04\\

\hline
\end{tabular}
}
\label{tab:modelnet40_kiss_mathcer_vs_ESM-ICP}
\end{table}

\begin{table}[h!]
\small
\centering
\caption{Comparative analysis of ESM-ICP and KISS-Matcher~\cite{lim2025icra-KISSMatcher} across 1000 different transformations ranging from $[-\pi,\pi]$ rotation wise and $[-1000,1000]$ translation wise. Here we use the stool dataset and in addition to these random transformation we corrupt $30 \%$ of the source dataset with random noise. The target dataset remains unchanged.}
\resizebox{0.50\textwidth}{!}{  
\begin{tabular}{|l|c|c|c|c|c|c|}
\hline
\textbf{Model:} & \textbf{MSE($R$)} & \textbf{RMSE($R$)} & \textbf{MAE($R$)} & \textbf{MSE($t$)} & \textbf{RMSE($t$)} & \textbf{MAE($t$)} \\
\hline
Stool (KISS-Matcher)~\cite{lim2025icra-KISSMatcher} & 0.073004 &  0.127731 &  0.102708 & 27542.029047  & 45.946623 & 41.252042 \\
\hline
Stool (\textbf{ESM-ICP (ours)})    & 0.073224 &  0.220009  &  0.179977  & 0.040046  & 0.164954   & 0.142406 \\

\hline
\end{tabular}
}
\label{tab:stool_kiss_mathcer_vs_ESM-ICP_30_per}
\end{table}
\section{Conclusions and Future Work} 
\label{Section_5} 
One of the core objectives addressed in this work is the accurate estimation of the transformation between the Source and Target point clouds under varying rotations and translations. We evaluated our proposed approach across multiple datasets, where the Source was generated by applying randomly sampled transformation matrices to the Target. The results consistently demonstrate that our method outperforms existing state-of-the-art techniques regarding alignment accuracy.

Furthermore, we tested the robustness of our approach under noise conditions, where the Source was corrupted by randomly injected noise of varying intensity and at arbitrary points. Even in such challenging scenarios, our method maintained superior performance compared to leading alternatives, successfully estimating transformations between the noisy Source and the clean Target.

We strongly encourage readers to explore our implementation, which is publicly available on our GitHub repository. Looking ahead, we intend to extend our method to SLAM applications, leveraging the estimated poses for robot localization and environment mapping for aerial and ground platforms. 

\section*{Acknowledgement}
This work is supported by NASA EPSCoR under grant number: 80NSSC24M0141, the U.S. National Science Foundation (NSF) under grant NSF-CAREER: 1846513, and the U.S. Army's Engineer Research and Development Center under grant number: W911NF-23-1-0186. 
The views, opinions, findings, and conclusions reflected in this publication are solely those of the authors and do not represent the official policy or position of NASA, NSF, and the US Army.



\end{document}